\documentclass[lettersize,journal]{IEEEtran}
\usepackage{algorithmic}
\usepackage[nocomma, short]{optidef} % for formulating optimization problems
\usepackage{amsmath,amsfonts,amsthm}
\usepackage{algorithm}
\usepackage{array}
\usepackage[caption=false,font=normalsize,labelfont=sf,textfont=sf]{subfig}
\usepackage{textcomp}
\usepackage{stfloats}
\usepackage{url}
\usepackage{verbatim}
\usepackage{graphicx}
\usepackage{cite}
\usepackage{multicol}
\usepackage{booktabs} % for professional tables
\usepackage{pifont}% for check marks
\usepackage[dvipsnames]{xcolor}
\usepackage[linecolor=Maroon,backgroundcolor=Maroon!25,bordercolor=Maroon]{todonotes}
\usepackage{framed}

\newtheorem{definition}{Definition}[section] 
\newtheorem{Assumption}{Assumption}[section]

\newtheorem{Proposition}{Proposition}[section]
\newtheorem{Requirement}{Requirement}[section]

\colorlet{shadecolor}{Maroon!25}

\newcommand\copyrighttext{%
	\footnotesize \textcopyright 2024 IEEE. Personal use of this material is permitted. Permission from IEEE must be obtained for all other uses, in any current or future media, including reprinting/republishing this material for advertising or promotional purposes, creating new collective works, for resale or redistribution to servers or lists, or reuse of any copyrighted component of this work in other works.}
\newcommand\copyrightnotice{%
	\begin{tikzpicture}[remember picture,overlay]
		\node[anchor=south,yshift=5pt] at (current page.south) {\fbox{\parbox{\dimexpr\textwidth-\fboxsep-\fboxrule\relax}{\copyrighttext}}};
	\end{tikzpicture}%
}

\usepackage{acronym}
\newacro{QP}{quadratic program}
\newacro{MIQP}{mixed-integer quadratic program}
\newacro{MILP}{mixed-integer linear program}
\newacro{MIP-DM}{mixed-integer programming-based vehicle decision maker}
\newacro{MINLP}{mixed-integer nonlinear program}
\newacro{MIMP}{mixed-integer motion planner}
\newacro{MIOCP}{mixed-integer optimal control problems}
\newacro{MI}{mixed-integer}
\newacro{MIP}{mixed-integer programming}
\newacro{NLP}{nonlinear program}
\newacro{NMPC}{nonlinear model predictive controller}
\newacro{EDS}{permutation equivariant deep set}
\newacro{REDS}{recurrent permutation equivariant deep set}
\newacro{LSTM}{long short-term memory}
\newacro{NN}{neural network}
\newacro{OCP}{optimal control problem}
\newacro{FF}{feed forward network}
\newacro{DNN}{deep neural network}
\newacro{FCF}{Frenet coordinate frame}
\newacro{CCF}{Cartesian coordinate frame}
\newacro{FP}{feasibility projector}
\newacro{SQP}{sequential quadratic programming}
\newacro{RL}{reinforcement learning}
\newacro{COCO}{\emph{combinatorial offline convex online}}
\newacro{RNN}{recurrent neural network}
\newacro{BB}{branch-and-bound}
\newacro{EDS}{equivariant deep set}
\newacro{AD}{automated driving}
\newacro{DM}{decision making}
\newacro{PCC}{Pearson Correlation Coefficient}
\newacro{TD}{training distribution}
\newacro{SD}{simulation distribution}

\newcommand{\mpc}{reference tracking \ac{NMPC}}
\newcommand{\expert}{expert \ac{MIP-DM}}
\newcommand{\safetyfilter}{{\ac{FP}}}
\newcommand{\nnmodule}{\ac{REDS} planner}

\newcommand{\slackqp}{soft-\ac{QP}}

\newcommand{\nndslstm}{{\ac{REDS} network}}

\newcommand{\rdnetCol}{\texttt{DEU\_Cologne-63\_5\_I-1}}

\newcommand{\optim}[1]{#1^*}
\newcommand{\estim}[1]{\hat{#1}}
\newcommand{\refer}[1]{\tilde #1}
\newcommand{\ub}[1]{\overline{#1}}
\newcommand{\lb}[1]{\underline{#1}}

\newcommand{\xPm}{x}
\newcommand{\xPmIni}{\estim{x}}
\newcommand{\xPmRef}{\refer{x}}

\newcommand{\XPmRef}{\refer{X}}
\newcommand{\UPmRef}{\refer{U}}

\newcommand{\XPmOpt}{X^*}

\newcommand{\uPm}{u}
\newcommand{\uPmRef}{\refer{u}}

\newcommand{\XPm}{X}

\newcommand{\UPm}{U}

\newcommand{\xExp}{x^{\mathrm{e}}}

\newcommand{\xCor}{x^\mathrm{s}}

\newcommand{\XExp}{X^{\mathrm{e}}}
\newcommand{\XNN}[1]{X^\mathrm{p #1}}

\newcommand{\UNN}[1]{U^\mathrm{p #1}}

\newcommand{\stateLon}{s}
\newcommand{\stateLat}{n}

\newcommand{\td}{t_\mathrm{d}}

\newcommand{\dimT}{T_\mathrm{f}}
\newcommand{\dimObs}{n_\mathrm{obs}}
\newcommand{\dimHor}{N}
\newcommand{\dimInp}{n_{u}}

\newcommand{\dimPm}{n_\mathrm{x}}

\newcommand{\dimContr}{n_{u}}

\newcommand{\dimFeatEq}{m_\mathrm{eq}}
\newcommand{\dimFeatUs}{m_\mathrm{us}}
\newcommand{\dimHidden}{m_\mathrm{h}}
\newcommand{\dimEnsemble}{n_\mathrm{e}}

\newcommand{\laneWidth}{d_\mathrm{lane}}

\newcommand{\binL}{\lambda}
\newcommand{\binY}{\gamma}

\newcommand{\Bin}{B}
\newcommand{\BinPred}{\hat{B}}

\newcommand{\BinL}{\Lambda}
\newcommand{\BinY}{\Gamma}

\newcommand{\setCar}{\mathcal{O}^\mathrm{car}}
\newcommand{\setSafe}{\mathcal{O}^\mathrm{safe}}
\newcommand{\setOut}{\mathcal{O}^\mathrm{out}}
\newcommand{\setFreeOut}{\mathcal{F}^\mathrm{out}}

\newcommand{\setFreeSafe}{\mathcal{F}^\mathrm{safe}}

\newcommand{\selMat}{P}
\newcommand{\posFrenet}{p}
\newcommand{\eqFeat}{\zeta^\mathrm{eq}}
\newcommand{\usFeat}{\zeta^\mathrm{us}}

\newcommand{\hEq}{h^\mathrm{eq}}
\newcommand{\hUs}{h^\mathrm{us}}

\newcommand{\feas}{\mu}
\newcommand{\tcomp}{t_\mathrm{comp}}
\newcommand{\costRatio}{\rho}

\newcommand{\weightLc}{w_\mathrm{lc}}

\newcommand{\weightRight}{w_\mathrm{rght}}
\newcommand{\weightDist}{w_\mathrm{dst}}
\newcommand{\weightHard}{w_\mathrm{h}}

\newcommand{\R}{\mathbb{R}}

\newcommand{\norm}[1]{\left\lVert #1 \right\rVert}

\newcommand{\ra}[1]{\renewcommand{\arraystretch}{#1}} 

\begin{document}

\title{Equivariant Deep Learning of Mixed-Integer Optimal Control Solutions for Vehicle Decision Making and Motion Planning}

\author{Rudolf Reiter,
	 Rien Quirynen,
	 Moritz Diehl,
	 Stefano Di Cairano,~\IEEEmembership{Senior Member,~IEEE}
        % <-this % stops a space

\thanks{R. Quirynen and S. Di Cairano are with 
Mitsubishi Electric Research Laboratories, Cambridge, MA, USA (e-mail: \{quirynen,dicairano\}@merl.com).}
\thanks{R. Reiter and M. Diehl are with 
the Unversity of Freiburg, 79110 Freiburg i. B., Germany (e-mail: \{rudolf.reiter,moritz.diehl\}@imtek.uni-freiburg.com).}
\thanks{Rudolf Reiter and Moritz Diehl were supported by EU via ELO-X 953348, by DFG via Research Unit FOR 2401, project 424107692 and 525018088 and by BMWK via 03EI4057A and 03EN3054B.}% <-this % stops a space
}



\maketitle
\copyrightnotice

\begin{abstract}
Mixed-integer quadratic programs~(MIQPs) are a versatile way of formulating vehicle decision making and motion planning problems, where the prediction model is a hybrid dynamical system that involves both discrete and continuous decision variables.
However, even the most advanced MIQP solvers can hardly account for the challenging requirements of automotive embedded platforms. Thus, we use machine learning to simplify and hence speed up optimization.
Our work builds on recent ideas for solving MIQPs in real-time by training a neural network to predict the optimal values of integer variables and solving the remaining problem by online quadratic programming.
Specifically, we propose a recurrent permutation equivariant deep set that is particularly suited for imitating MIQPs that involve many obstacles, which is often the major source of computational burden in motion planning problems.
Our framework comprises also a feasibility projector that corrects infeasible predictions of integer variables and considerably increases the likelihood of computing a collision-free trajectory.
We evaluate the performance, safety and real-time feasibility of decision-making for autonomous driving using the proposed approach on realistic multi-lane traffic scenarios with interactive agents in SUMO simulations.
\end{abstract}


\acresetall
\section{Introduction}
\IEEEPARstart{D}{ecision}-making and motion planning for automated driving is challenging due to several reasons~\cite{Paden2016}.
First, in general, even formulations of deterministic, two-dimensional motion planning problems are PSPACE-hard~\cite{Reif1979, LaValle2006}.
Second, (semi-)autonomous vehicles operate in highly dynamic environments, thus requiring a relatively high control update rate.
Finally, there is always uncertainty that stems from model mismatch, inaccurate measurements as well as other drivers' unknown intentions.
The complexity of motion planning and \ac{DM} for automated driving and its real-time requirements in resource-limited automotive platforms~\cite{DiCairano2018tutorial} requires the implementation of a multi-layer guidance and control architecture~\cite{Paden2016, Guanetti2018}.

Based on a route given by a navigation system, a decision-making module decides what maneuvers to perform, such as lane changing, stopping, waiting, and intersection crossing. Given the outcome of such decisions, a motion planning system generates a trajectory to execute the maneuvers, and a vehicle control system computes the input signals to track it. Recent work \cite{Quirynen2023} presented a \ac{MIP-DM}, which simultaneously performs maneuver selection and trajectory generation by solving a \ac{MIQP} at each time instant. In this paper, we present an algorithm to implement \ac{MIP-DM} based on supervised learning and \ac{SQP}, to compute a collision-free and close-to-optimal solution
with a considerably reduced online computation time compared to advanced \ac{MIQP} solvers.

The presented approach consists of an \emph{offline} supervised learning procedure and an \emph{online} evaluation step that includes a \safetyfilter{}. In the \emph{offline} procedure, expert data is collected by computing the exact solutions of the \ac{MIQP} for a large number of samples from a distribution of parameter values in the \ac{MIP-DM}, including for example states of the autonomous vehicle, its surrounding traffic environment, and speed limits.
Along the paradigm of~\cite{Abhishek2022a, Bertsimas2021}, a \ac{NN} is trained with the collected expert data to predict the binary variables that occur within the \ac{MIQP}, which are the main source of the computational complexity of \acp{MIQP}. 
A novel \ac{NN} architecture, referred to as \ac{REDS}, is proposed that exploits key structural domain properties, such as permutation equivariance related to obstacles and recurrence of the time series.
\IEEEpubidadjcol

In the \emph{online} evaluation step, the \ac{NN} predicts the optimal values of the binary variables in the \ac{MIQP}.
After fixing these, the resulting problem becomes a convex \ac{QP} that can be solved efficiently.
To account for potentially wrong predictions, the \ac{QP} is formulated with slack variables (\slackqp{}).
The \slackqp{} solution is forwarded to a \safetyfilter{} to correct any infeasibilities, implemented by a \ac{NLP} with smooth, but concave obstacle constraints. Such convex-concave \acp{NLP} can be solved efficiently using an \ac{SQP} algorithm~\cite{Acados2021,Tran2012}.
To further increase the likelihood of finding a feasible and possibly optimal solution, an ensemble of \acp{NN} is trained and evaluated, and the ``best'' solution is selected at each time step.

The overall performance of the proposed method is compared against the advanced \ac{MIQP} solver \texttt{gurobi}~\cite{Gurobi2023}, alternative neural network architectures~\cite{Abhishek2022} and evaluated in high fidelity closed-loop simulations using \texttt{SUMO}~\cite{Sumo2018} and \texttt{CommonRoad}~\cite{Althoff2017}.

\subsection{Related Work}
\label{sec:related_work}
This work lies at the intersection of three research areas, i.e., geometric deep learning, \ac{MIP} and motion planning for \ac{AD} (cf., Fig.~\ref{fig:related}).
\begin{figure}
	\begin{center}
		\begin{small}
		\def\svgwidth{.32\textwidth}
		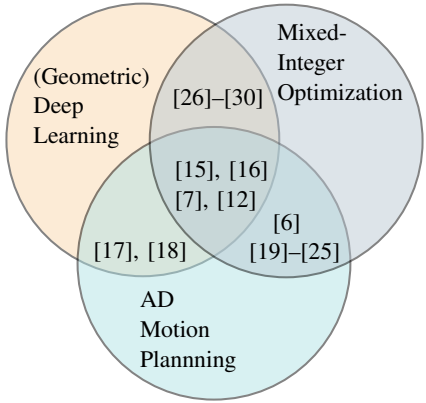
		\end{small}
		\caption{Categorical overview of related work. The research domain of the proposed approach is located at the intersection of three areas.}
		\label{fig:related}
	\end{center}
\end{figure}
% motion planning
Motion planning problems are solved by algorithms that can handle combinatorial complexity and dynamic feasibility under real-time computation limits~\cite{Claussmann2020, Reda2024}.

Motion primitives are appealing due to their simplicity and alignment with the road geometry~\cite{Sheckells2017}. However, the motion plans are usually sub-optimal or conservative.

In several works, graph search is performed on a discretized state space~\cite{Paden2016}.
Probabilistic road maps~\cite{Kavraki1996} and rapidly exploring random trees~\cite{Oktay2017} are common in highway motion planning. Nonetheless, they suffer from the curse of dimensionality, a poor connectivity graph, and no repeatability~\cite{Claussmann2020}.
By using heuristics, $\mathrm{A}^*$~\cite{Ajanovic2018} aims to avoid the problem of high-dimensional discretization. However, choosing an appropriate heuristic is challenging and graph generation in each iteration is time-consuming~\cite{Paden2016}. 

For some conservative simplifications related to highway driving, the state space can be decomposed into spatio-temporal driving corridors~\cite{Bender2015,Miller2018,Li2022}.
Such non-convex regions can be further decomposed into convex cells~\cite{Deolasee23} or used in a sampling-based planner~\cite{Li2022}. For these cases, the performance is limited due to the required over-approximations.

Derivative-based numerical optimization methods can successfully solve problems in high-dimensional state spaces in real-time~\cite{Diehl2005}. However, these approaches are often restricted to convex problem structures or sufficiently good initial guesses.
While highway motion planning and \ac{DM} is highly non-convex, by introducing integer variables, the problem can be formulated as an \ac{MIQP}~\cite{Junghee2015,Xiangjun2016,Esterle2020,Kessler2020,Eiras2022,Li2022,Kessler2023,Quirynen2023}.

A common problem is the significant computation time of \ac{MIQP} solvers.
Authors mitigate the problem by either scaling down the problem size~\cite{Xiangjun2016, Reiter2021,Eiras2022,Quirynen2023}, neglecting real-time requirements~\cite{Esterle2020} or accepting the high computation time for non-real-time simulations while leaving real-time feasibility open for future work~\cite{Junghee2015,Kessler2020,Kessler2023}.

Several recent works use deep learning to accelerate \ac{MIQP} solutions. 
One strategy is to improve algorithmic components within an online \ac{MIQP} solver~\cite{Masti2019, Nair2021, Khalil2022}.
The authors in~\cite{Masti2019} use deep learning to warm-start \ac{MIQP} solvers, which however cannot lower the worst-case computation time.
The authors in~\cite{Nair2021,Khalil2022} propose a learning strategy to guide a \ac{BB} algorithm.
For solving \acp{MILP}, the authors in~\cite{Russo2023} use \acp{NN} for a custom solver to achieve similar computation times as commercial solvers, which, however, may still be large.
Another strategy is based on supervised or imitation learning, using \ac{MIQP} solutions as expert data to train function approximators offline~\cite{Bertsimas2022,Chenyang2020,Srinivasan2021}.
The authors in~\cite{Bertsimas2022} show how the classification of binary variables in \acp{MIQP} can produce high-quality solutions with a low computation time. 
Recent work~\cite{Abhishek2022} extends supervised learning for \ac{MIQP}-based motion planning~\cite{Abhishek2022a} using a \ac{RNN} and presolve techniques to increase the likelihood of predicting a feasible solution.
Due to the \ac{RNN}, the results in~\cite{Abhishek2022} scale well with the horizon length, but not yet with the number of obstacles, as shown later.

% all domains combined
None of the works that use deep learning to accelerate \ac{MIQP}-based motion planning consider or leverage geometric deep learning, particularly, permutation equivariance or invariance, to decrease the network size and increase the performance. As one consequence, they do not allow variable input and variable output dimensions, e.g., corresponding to a variable number of obstacles.
However, some recent works successfully use geometric deep learning for related tasks in \ac{AD}, e.g., the prediction of other road participants~\cite{Zhou2021, Eunsan2021}, or within \ac{RL}~\cite{Huegle2019}. 

Other techniques for \ac{DM} exclusively use \acp{NN}, without solving optimization problems online, e.g., using deep \ac{RL}~\cite{Szilard2022}. These approaches usually require more data, are less interpretable and may lack safety without additional safety layers, due to the approximate nature of the \ac{NN} output~\cite{Claussmann2020}.

\subsection{Contributions}
\label{sec:contribution}
Our main contribution is a \ac{REDS} for the \ac{NN} predicting integer variables of \acp{MIQP}, which is particularly suited for learning time-series and obstacle-related binary variables in motion planning problems.
In particular, the \ac{REDS} enables the \ac{NN} to predict binary variables for collision avoidance concerning a varying number of obstacles, and the predicted solution is the same regardless of the order in which the obstacles are provided.
Our proposed framework includes an ensemble of \acp{NN} in combination with a feasibility projector (\ac{FP}) that increases the likelihood of computing a collision-free trajectory.
Compared to the state-of-the-art, we show that our method improves the prediction accuracy, adds permutation equivariance, allows for variable number of obstacles and horizon length, and generalizes well to unseen data, such as several obstacles not present in the training data.
As a final contribution, we demonstrate the performance of a novel integrated planning system, which is further referred to as \nnmodule{}. The \nnmodule{} uses an ensemble of \acp{REDS}, a selection of the best \slackqp{} solution and the \ac{FP} for real-time vehicle decision-making and motion planning. We present closed-loop simulation results with reference tracking by a \ac{NMPC} for realistic traffic scenarios, using interactive agents in \texttt{SUMO}~\cite{Sumo2018}, and a high-fidelity vehicle model and the problem setup provided by~\texttt{CommonRoad}~\cite{Althoff2017}.\\

\subsection{Preliminaries and Notation}
The notation $I(n)=\{z \in \mathbb{N} | 0 \leq z \leq n\}$ is used for index sets and $\mathbb{B}=\{0,1\}$.
% for the set of binary variables.
Throughout this paper, the attributes of \emph{equivariance (invariance)} exclusively refer to \emph{permutation equivariance (invariance)} with the following definition.
\begin{definition}
	Let $f(\usFeat):\mathbb{X}^M\rightarrow\mathbb{Y}$ be a function on a set of variables $\usFeat=\{\usFeat_1,\ldots,\usFeat_M\}\in \mathbb{X}^M$ and let $\mathcal{G}$ be the permutation group on $\{1,\ldots,M\}$. The function $f$ is \textbf{permutation invariant}, if $f(g \cdot \usFeat) = f(\usFeat)$ for all $g \in \mathcal{G}, \usFeat \in \mathbb{X}^M$. \label{def:invariant}
\end{definition}
\begin{definition}
	Let $f(\eqFeat):\mathbb{X}^N\rightarrow\mathbb{Y}^N$ be a function on a set of variables $\eqFeat=\{\eqFeat_1,\ldots,\eqFeat_N\}\in \mathbb{X}^N$ and let $\mathcal{G}$ be the permutation group on $\{1,\ldots,N\}$. The function $f$ is \textbf{permutation equivariant}, if $f(g \cdot \eqFeat) = g \cdot f(\eqFeat)$ for all $g \in \mathcal{G}, \eqFeat \in \mathbb{X}^N$. \label{def:equivariant}
\end{definition}
Features that are modeled without structural symmetries are referred to as \emph{unstructured features}.
Lower and upper bounds on decision variables $x$ are denoted by $\lb{x}$ and $\ub{x}$, respectively.
The all-one vector~$[1,\ldots,1]^\top$ of size $n$ is $\boldsymbol{1}^\top_n$.
The notation $f(x;y)$ in the context of optimization problems indicates that function $f(\cdot)$ depends on decision variables $x$ and parameters $y$.
Lower case letters $x\in\R^n$ refer to scalars or vectors of size $n$ and their upper case version $X\in\R^{n\times N}$ refer to the matrix associated with a sequence of those vectors along a time horizon $N$.
\emph{Obstacles} normally refer to surrounding vehicles.

\section{Problem Setup and Formulation}
\label{sec:minlp}
In this work, an autonomous vehicle shall drive \emph{safely} along a multi-lane road, while obeying the traffic rules.
We consider decision-making based on \ac{MIP} that determines the driving action and a reference trajectory for the vehicle control to follow.
The definition of \emph{safety} can be ambiguous~\cite{Brunke2022}. We use the term \emph{safe} to refer to satisfying hard collision avoidance constraints concerning known obstacles' trajectories.
Our problem setup relies on the following assumptions.
\begin{Assumption}
	There exists a prediction time window along which the following are known:
	\begin{enumerate}
		\item the predicted position and orientation for each of the obstacles in a neighborhood of the ego vehicle up to a sufficient precision,
		\item the high-definition map information, including center lines, road curvature, lane widths, and speed limits.
		% \hfill\QEDopen
	\end{enumerate}
	\label{as:problem}
\end{Assumption}
Assumption~\ref{as:problem} requires the vehicle to be equipped with on-board sensors and perception systems~\cite{Brummelen2018}, an obstacle prediction module~\cite{Paden2016,Guanetti2018}, the high definition map database, either on-board or obtained through vehicle-to-infrastructure~(V2I) communication~\cite{ersal2020}.
The effect of prediction inaccuracies and uncertainty on vehicle safety is outside the scope of the present paper, but relevant work can be found in~\cite{Paden2016,Guanetti2018} and references therein.
\begin{Assumption}
	We assume a localization with cm-level accuracy at each sampling time.
	\label{as:gps}
\end{Assumption}
Assumption~\ref{as:gps} is possible thanks to recent advances in GNSS that achieves cm-level accuracy at a limited cost~\cite{Berntorp2020,Greiff2023}.

We propose a \ac{DM} that satisfies the following requirements.
\begin{Requirement}
	\label{req:coll}
	The \ac{DM} must plan a sequence of maneuvers, possibly including one or multiple lane changes, and a corresponding collision-free trajectory to make progress along the vehicle's future route with a desired nominal velocity.
\end{Requirement} 
\begin{Requirement}
	\label{req:feas}
	The \ac{DM} yields kinematically feasible trajectories, satisfying vehicle limitations and traffic rules (e.g., speed limits).
\end{Requirement} 
\begin{Requirement}
	\label{req:eq}
	The \ac{DM} must be agnostic to permutations of surrounding vehicles, i.e., the plan must be consistent, irrespective of the order the vehicles are processed (cf. Def.~\ref{def:equivariant}).
\end{Requirement} 
\begin{Requirement}
	\label{req:comp_t}
	The \ac{DM} must hold a worst-case computation time lower than the real-time planning period $t_\mathrm{p}$, with a target value of $t_\mathrm{p}\leq0.2\mathrm{s}$.
\end{Requirement}
\begin{Requirement}
	\label{req:var}
	At any time, the \ac{DM} must satisfy all requirements, regardless of the number of surrounding vehicles.
\end{Requirement}
Based on Assumption~\ref{as:problem}, 
Req.~\ref{req:coll}-\ref{req:eq} can be met by the \ac{MIP-DM} from~\cite{Quirynen2023}.
However, the crucial Req.~\ref{req:comp_t}-\ref{req:var} may be difficult to meet by the \ac{MIP-DM}, when executing on embedded control hardware with limited computational resources and memory~\cite{DiCairano2018tutorial}.
The main focus of this work is approximating the \ac{MIP-DM} with a suitable framework that satisfies all Req.~\ref{req:coll}-\ref{req:var}.
Notably, Req.~\ref{req:eq} is trivially true for the \ac{MIP-DM}~\cite{Quirynen2023}.
However, this is not the case for \ac{NN}-based planners unless the architecture is invariant to permutations~\cite{Zaheer2017}.

\subsection{Nominal and Learning-based Architecture}
\label{sec:architecture}
A hierarchical control architecture (cf., Fig~\ref{fig:sw_arch}) of a \ac{DM} is proposed, followed by a \mpc{}. 
Within the architecture of Fig.~\ref{fig:sw_arch}, the \expert{} is used as a benchmark.
We refer to the module aiming to imitate the \expert{} as the \nnmodule{}.
It comprises an ensemble of \acp{NN} that predict the binary variables of the \ac{MIQP}, as used in the \expert{}.
For each \ac{NN}, a \slackqp{} is solved with the binary variables fixed to the predictions and softened obstacle avoidance constraints.
This yields a set of candidate trajectories, 
and a selector evaluates their costs and picks the least sub-optimal one.
Finally, the \safetyfilter{} projects the solution guess to the feasible set to ensure that the constraints are satisfied.
The \safetyfilter{} solves an \ac{NLP}, minimizing the error with respect to the best candidate solution and subject to ellipsoidal collision avoidance constraints.
The reference provided by the \nnmodule{} or the \expert{} is tracked by the \mpc{} at a high sampling rate which includes collision avoidance constraints. This adds additional safety and leverages the requirements on the \nnmodule{} on conservative constraints related to uncertain predictions.\\

The individual modules differ in motion prediction models, their objective, obstacle avoidance constraints, the approximate computation time~$\tcomp$, the final horizon~$\dimT$ and the discretization time~$\td$, cf., Tab.~\ref{table:overview_opti}.

\begin{figure}
	\begin{center}
		\fontsize{7pt}{7pt}
		\def\svgwidth{.48\textwidth}
		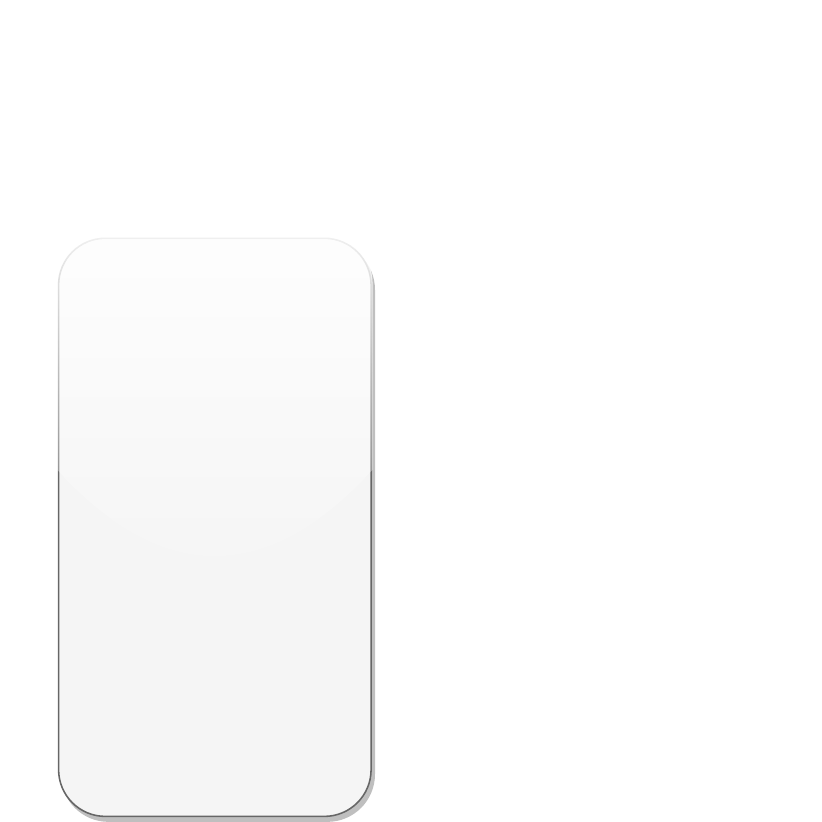
		\caption{Planning and closed-loop simulation architecture. The \expert{} is imitated by the \nnmodule{}, which uses an ensemble of $\dimEnsemble$ \acp{NN} to predict values of the binary variables~$\{\BinPred^{1},\ldots,\BinPred^{\dimEnsemble}\}$. A \slackqp{} solves a formulation of the \expert{} with binary variables fixed to the prediction. The lowest cost solution~$\XNN{}$ is chosen by a Selector and corrected by the \safetyfilter{}. 
  		A \mpc{} with obstacle avoidance tracks the corrected solution $X^\mathrm{s}$.}
		\label{fig:sw_arch}
	\end{center}
\end{figure}
\begin{table*}
	\centering
	\ra{1.2}
	\begin{tabular}{@{}lllllllll@{}}
		\addlinespace
		\toprule
		Module & Opt. Class & Model & Objective & Obstacle Avoidance &$\tcomp$&$\dimT$ &$\td$&Comment\\
		\midrule
		\expert & MIQP & point-mass & global economic &  hyper-planes $\setOut$& $\sim2$s &$10$s&$0.2$s& high computation time\\
		\slackqp & QP & point-mass & global economic & fixed hyper-planes $\setOut$& $\sim3$ms &$10$s&$0.2$s& possibly unsafe\\
		\safetyfilter & NLP & point-mass & tracking (\slackqp{}) & ellipsoids $\setSafe$ & $\sim10$ms&$10$s&$0.2$s& safe projection\\
		tracking \ac{NMPC} & NLP & kinematic & tracking (\safetyfilter{}) & ellipsoids $\setSafe$& $\sim2$ms &$1$s &$0.05$s&\\
		\bottomrule
		\vspace{0.5px}
	\end{tabular}
	\caption{Overview of optimization problems solved within the individual modules.}
	\label{table:overview_opti}
\end{table*}

\section{Expert Motion Planner}
\label{sec:mimp}
We model the vehicle state in a curvilinear coordinate frame, which is defined by the curvature $\kappa(\stateLon)$ along the reference path~\cite{Reiter2021a}.
The state vector~$\xPm=[\stateLon, \stateLat, v_\stateLon, v_\stateLat]^\top\in \R^{\dimPm}$ includes the position~$\posFrenet=[\stateLon,\stateLat]^\top\in \R^2$, with the longitudinal position $\stateLon$, the lateral position~$\stateLat$, the longitudinal velocity~$v_{\stateLon}$ and the lateral velocity $v_{\stateLat}$, where $\dimPm=4$.
The control vector $\uPm=[a_\mathrm{s},a_\mathrm{n}]^\top\in \R^{\dimContr}$ comprises the longitudinal and lateral acceleration in the curvilinear coordinate frame, where $\dimContr=2$.
The discrete-time double integrator dynamics are written as $\xPm_{i+1}=A\xPm_{i}+B\uPm_{i}$, using the discretization time~$t_\mathrm{d}$, where $A\in\R^{\dimPm\times\dimPm}$ and $B\in\R^{\dimPm\times\dimContr}$.
By considering $\dimHor$ prediction steps, the prediction horizon can be computed by $\dimT=\dimHor t_\mathrm{d}$.
Constraints $v_\stateLat\leq \ub{\alpha}\, v_\stateLon$ and $v_\stateLat\geq \lb{\alpha}\, v_\stateLon$ account for limited lateral movement of the vehicle with bounds $\lb{\alpha}, \ub{\alpha}$ that can be computed based on a maximum steering angle $\ub{\delta}$ and the maximum absolute signed curvature $\ub{\kappa}=\kappa(s^*)$, with~$s^*=\arg\max_{0\leq s \leq \ub{s}} |\kappa(s)|$ along a lookahead distance~$\ub{s}$~\cite{Quirynen2023}.
% Discussion about models: https://gitlab.lrz.de/tum-cps/commonroad-vehicle-models/-/blob/master/vehicleModels_commonRoad.pdf
An acceleration limit $\ub{a}_\mathrm{fric}$ of the point-mass model formulated as box constraints $||u||_\infty\leq\ub{a}_\mathrm{fric}$ inner-approximate tire friction constraints related to Kamm's circle~\cite{Althoff2017}.
Since we formulate our model in the Frenet coordinate frame, the lateral acceleration bounds $\ub{a}_n, \lb{a}_n$ are modified, based on the centrifugal acceleration, resulting in $\ub{a}_n=\ub{a}_\mathrm{fric} + \ub{\kappa} v_{\stateLon,0}^2$ and $\lb{a}_n=-\ub{a}_\mathrm{fric} + \ub{\kappa} v_{\stateLon,0}^2$.
The fixed parameter $d_\mathrm{bnd}$ is used as the safety distance to the road boundary, bounds $\ub{u},\lb{u}$ for acceleration limits, and $\ub{v}_\stateLon, \ub{v}_\stateLat$ for the maximum velocity.
In order to account for the road width, bounds $\lb{n}\le \stateLat \le\ub{n}$ on the lateral position are used. 
The model used within \expert{} is able to approximate the dynamics in a variety of situations~\cite{Quirynen2023}. However, certain maneuvers, such as sharp turns, may require additional modeling concepts~\cite{Kessler2020}.

\subsection{Collision Avoidance Constraints}
Collision avoidance constraints for $n_\mathrm{obs}$ obstacles are formulated by considering the ego vehicle as a point mass and using a selection matrix~$\selMat\in \R^{2\times n_x}$, with $\posFrenet=\selMat x$ that selects the position states~$\posFrenet$ from the states~$x$. 
The true occupied obstacle space $\setCar_j \subseteq\R^2$, for $j\in I(\dimObs-1)$, is expanded to include all possible configurations where the ego and obstacle vehicle are in a collision in the curvilinear coordinate frame, resulting in an ellipsoid~$\setSafe_j$, cf. Fig.~\ref{fig:vehicle_shape}.
\begin{figure}
	\begin{center}
		\def\svgwidth{.4\textwidth}
		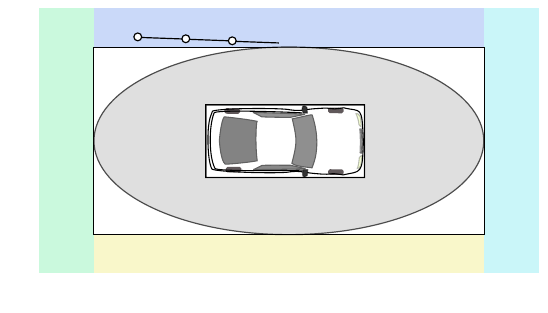
		\caption{Vehicle over-approximations. All trajectories outside of $\setSafe$ are considered free of collision. The \expert{} plans with the most conservative obstacle set $\setOut$. The ellipsoidal smooth over-approximation $\setSafe$ is used within the \safetyfilter{} and the \ac{NMPC}. The four colored regions are uniquely determined by the binary variables $\binY$, where in each region exactly one binary variable is equal to one.}
		\label{fig:vehicle_shape}
	\end{center}
\end{figure}
The \expert{} uses a road-frame-aligned rectangular constraint $\selMat \xPm_i\notin\setOut_j$, for all $ i \geq 0, j\in I(\dimObs-1)$, which over-approximates the ellipsoidal set $\setSafe_j$, leading to additional robustness of the multi-layer control architecture, with $\setCar_j\subseteq\setSafe_j\subseteq\setOut_j\subseteq\R^2$ (cf., Fig.~\ref{fig:sw_arch} and Tab.~\ref{table:overview_opti}).
For each obstacle $j$, the rectangular shape of $\setOut_j$ can be characterized by the boundaries $d^\top=[s_\mathrm{f}, s_\mathrm{b}, n_\mathrm{l}, n_\mathrm{r}]^\top$.
Four collision-free regions $k \in \{\mathrm{f,b,l,r}\}$ around an obstacle are defined, where each region can be expressed by a convex set~$A^{k}(d)\,\posFrenet\leq b^{k}(d)$, with either one for $k\in\{f,b\}$ or three half-space constraints for $k\in\{l,r\}$, cf., Fig.~\ref{fig:vehicle_shape}.
Four binary variables $\binY=[\binY_\mathrm{f},\binY_\mathrm{b},\binY_\mathrm{l},\binY_\mathrm{r}]^\top\in\mathbb{B}^4$
are used with a big-M formulation to ensure that the vehicle is inside one of the four regions.
Therefore, with $\boldsymbol{1}$ in the according dimension of either $\boldsymbol{1}_1$ for $k\in\{f,b\}$ or  $\boldsymbol{1}_3$ for $k\in\{l,r\}$, the collision-free set is
% The obstacle-free set can be defined by 
\begin{align}
\label{eq:cobstacle_constr}
\begin{split}	
	&\setFreeOut(d, \ub{\sigma}) = 
	\Big\{ (\posFrenet,\binY,\sigma) \in (\R^{2},\mathbb{B}^4,\R )\Big| \forall k \in\{\mathrm{f,b,l,r}\}:\\
	& A^{k}(d)\,\posFrenet\leq b^{k}(d)+\boldsymbol{1}(1-\binY_k) M + \boldsymbol{1}(\sigma - 1)\ub{\sigma}_k, \;
	 \boldsymbol{1}^{\top}_{4}\binY=1
	\Big\}.
\end{split}
\end{align}
The values $\ub{\sigma}_k$ define additional safety margins for each of the four collision-free regions. The slack variable $\sigma\in\R$ is bounded $0\le\sigma\leq 1$, such that only points in the exterior of $\setOut_j$ satisfy the condition in~\eqref{eq:cobstacle_constr}.
A model predicts the future positions of the obstacle boundaries $d_i^j$ for each obstacle $j\in I(\dimObs-1)$ at time steps $i \in I(\dimHor)$ along the horizon.

\subsection{\ac{MIQP} Cost Function}
The \ac{MIQP} cost comprises quadratic penalties for the state vector $\xPm\in \R^{\dimPm}$ with weight $Q\in\R^{\dimPm\times\dimPm}$, and the control vector $\uPm\in\R^{\dimInp}$ with weight $R\in\R^{\dimInp\times\dimInp}$, for tracking of their references $\xPmRef$ and $\uPmRef$, respectively. 
The reference $\xPmRef_i=[\refer{\stateLon}_i,\refer{\stateLat}_{i},\refer{v}_\stateLon,0]^\top$ at time step $i$ is determined by the desired velocity $\refer{v}_\stateLon$, as well as by the binary control vector $\lambda=[\lambda^\mathrm{up},\lambda^\mathrm{down}]^\top \in \mathbb{B}^{2}$.
% , with $\dimBinL=2$.
These binary variables are used to determine lane changes at time step $i$, resulting in the road-aligned lateral reference $\refer{X}_n=[{\refer{\stateLat}}_{0},\ldots,{\refer{\stateLat}}_{N}]^\top\in\R^{N+1}$
always being the center of the target lane.
The longitudinal position $\refer{\stateLon}_i$ follows from the velocity $\refer{v}_\stateLon$.
The \ac{MIQP} cost function reads
\begin{align}
	\label{eq:miqp_cost}
	\begin{split}
	%J^\mathrm{e}(&\XPm,\XPmRef,\UPm,\BinL,\Sigma) =\\
	& \sum_{i=0}^{\dimHor}\norm{\xPm_{i}-\xPmRef_{i}}_{Q}^2 + 
	\sum_{i=0}^{\dimHor-1}\norm{\uPm_{i}-\uPmRef_{i}}_{R}^2 +\\
	& \weightLc \sum_{i=0}^{\dimHor-1} \boldsymbol{1}^\top_2\binL_i + 
	 \weightRight \sum_{i=0}^{\dimHor} n_i + 
	 \weightDist  \sum_{j=0}^{\dimObs-1}\sum_{i=0}^{\dimHor} (\sigma_i^j)^{2}, \\
	\end{split}
\end{align}
including a penalty with weight $\weightRight>0$ to minimize deviations from the right-most lane,
a weight $\weightLc>0$ to penalize lane changes, and a weight $\weightDist>0$ penalizing slack variables to avoid being too close to any obstacle.

\subsection{Parametric \ac{MIQP} Formulation}
\label{sec:MIQP_formulation}
For a horizon of $\dimHor$ steps, the binary \ac{MIQP} decision variables are $\BinL=[\binL_0,\ldots,\binL_{\dimHor-1}]\in\mathbb{B}^{2 \times \dimHor}$ and $\BinY=[\binY_0^0,\binY_1^0,\ldots,\binY_\dimHor^{\dimObs-1}]\in\mathbb{B}^{4 \times \dimObs(\dimHor+1)}$, the real-valued states are $\XPm=[\xPm_{0},\ldots,\xPm_{\dimHor}]\in\R^{4 \times (\dimHor+1)}$, the control inputs are $\UPm=[\uPm_{0},\ldots,\uPm_{\dimHor-1}]\in\R^{2 \times \dimHor}$ and the slack variables are $\Sigma=[\sigma_0^0,\sigma_1^0,\ldots ,\sigma_\dimHor^{\dimObs-1}] \in \R^{\dimObs(\dimHor+1)}$.
Hard linear inequality constraints are
\begin{align*}
	\begin{split}
		&\{H(\XPm,\UPm)\geq0\} \;\Leftrightarrow \\
		\hspace{-3mm}\{ \XPm,\UPm |\; &  \lb{u}\leq \uPm_i\leq \ub{u},\quad i \in I(\dimHor-1), \\
        &  \lb{x}\leq \xPm_i\leq \ub{x}, \quad
		\lb{\alpha}\, v_{s,i} \leq v_{\mathrm{n},i}\leq \ub{\alpha}\, v_{\mathrm{s},i},\quad i \in I(\dimHor)
		\}.
	\end{split}
\end{align*}
The parameters $\Pi=(\xPmIni,\refer{v}_\stateLon,n_\mathrm{lanes}, D)$ include the initial ego state $\xPmIni$, the desired velocity $\refer{v}_\stateLon$, the number of lanes $n_\mathrm{lanes}$ and the time-dependent obstacle bounds $D$, where 
\begin{equation*}
	D=\big\{d^j_i\, \big|\, \forall j\in I(\dimObs-1), \forall i\in I(\dimHor)\big\}.
\end{equation*}
The parametric \ac{MIQP} solved within each iteration of the \ac{MIP-DM} is
\begin{mini!}[3]
	{\XPm,\UPm,\refer{X}_n,\BinL,\BinY,\Sigma}			
	{J^\mathrm{e}(\XPm,\UPm,\refer{X}_n,\BinL,\Sigma)}
	{\label{eq:MIQP}} 
	{} 
	\addConstraint{\xPm_0}{= \xPmIni,\qquad H(\XPm,\UPm)\geq0,}{0\le\Sigma\le1},
    \addConstraint{{\refer{\stateLat}}_{i+1}}{={\refer{\stateLat}}_{i}+\laneWidth\lambda^\mathrm{up}_i-\laneWidth\lambda^\mathrm{down}_i,}    {}
	\label{eq:miqp_line_ref_prog}
 % \begin{bmatrix}0\\\laneWidth{}\\0\\0\end{bmatrix},
	\addConstraint{\xPm_{i+1}}{= A\xPm_{i}+B\uPm_{i},}{i\in I(\dimHor-1)},
	\addConstraint{(\selMat\xPm_i,\binY_i^j,\sigma_i^j)}{\in \setFreeOut(d_i^j, \ub{\sigma}_i^j), }{i \in I(\dimHor)\nonumber},
	\addConstraint{}{}{j \in I(\dimObs-1)},
\end{mini!}
where the cost $J^\mathrm{e}(\cdot)$ is defined in~\eqref{eq:miqp_cost}, and $\refer{\stateLat}_{0}$ is the lateral position of the center of the desired lane.
An \ac{MIQP} solving~\eqref{eq:MIQP} is used as an ``expert'' to collect supervisory data, i.e., feature-label pairs $(\Pi,\optim{\Bin})$, where $\optim{\Bin}$ is the optimal value of binary variables, cf., Fig.~\ref{fig:sw_arch}.
For the closed-loop evaluation of the \expert{}, the optimizer~$\XPmOpt$ is used as the output of the \expert{}~$\XExp$.

\section{Scalable Equivariant Deep Neural Network}
\label{sec:nn}
Because the \ac{MIQP}~\eqref{eq:MIQP} is computationally demanding to solve in real-time, especially for long prediction horizons and a large number of obstacles, we propose a novel variant of the \ac{COCO} algorithm~\cite{Abhishek2022a} to accelerate \ac{MIQP} solutions using supervised learning. We train a \ac{NN} to predict binary variables and then solve the remaining convex \ac{QP} online, after fixing the binary variables.
% in the \ac{MIQP} (\fixQP{}).
The \ac{MIQP}~\eqref{eq:MIQP} comprises $4\dimObs\dimHor$ structured binary variables related to obstacles and $2\dimHor$ lane change variables. 
We refer to the latter as \emph{unstructured} binary variables because, differently from the others, there is no specific relation besides recurrence among them.
% Most of the binaries are related to obstacle avoidance which is, therefore, the focus of our \ac{NN} design.

In the following, we first describe the desired properties of the \ac{NN} prediction in Sec.~\ref{sec:nn_properties} and review the classification of binary variables in Sec.~\ref{sec:nn_classification}.
Then, in Sec.~\ref{sec:nn_equivariance}, we introduce one of our main contributions, the \ac{REDS} to achieve the desired properties. 
Finally, we show in Sec.~\ref{sec:nn_ensemble} how to use an ensemble of \acp{NN} to generate multiple predictions.

\subsection{Desired Predictor Properties for Motion Planning}
\label{sec:nn_properties}
The desired properties of the predictor may be divided into performance, i.e., general metrics that define the \ac{NN} prediction performance, and structural properties, i.e., structure-exploiting properties related to Requirements~\ref{req:coll}-\ref{req:var}.

\subsubsection{Performance}
The prediction performance is quantified by the likelihood of predicting a feasible solution $\feas$, and a measure of optimality $\costRatio$.
However, supervised learning optimizes accuracy, i.e., cross-entropy loss, which was shown to correlate well with feasibility~\cite{Bertsimas2021,Abhishek2022,Abhishek2022a}. In fact, we evaluated the \ac{PCC} for our experiments and obtained a \ac{PCC} of~$0.81$ for the correlation between the training loss and the infeasibility and a \ac{PCC} of~$0.88$ between accuracy and feasibility.
In addition, the computation time~$\tcomp$ to evaluate the \ac{NN} is important for real-time feasibility. We aim for~$\tcomp$ to be very small compared to the \ac{MIQP} solution time, and~$\tcomp < t_\mathrm{p}\leq0.2\mathrm{s}$ (see Req.~\ref{req:comp_t}). Finally, the memory footprint of the \ac{NN} should be small for implementation on embedded microprocessors~\cite{DiCairano2018tutorial}.

\subsubsection{Structural}
\label{sec:struct_properties}
First, the \nnmodule{} needs to operate on a variable number of obstacles, which is required in real traffic scenarios, see Req.~\ref{req:var}.
Second, to comply with Req.~\ref{req:eq}, obstacle-related predictions need to be equivariant to permutations on the input, 
see Def.~\ref{def:equivariant}. For unstructured binary variables, the predictions should be permutation invariant, see Def.~\ref{def:invariant}.
Third, again relating to Req.~\ref{req:var}, the \ac{NN} architecture is expected to generalize to unseen data.
Particularly, it should provide accurate predictions for several obstacles that may not be present in the training data.
Finally, the \ac{NN} should predict multiple guesses to increases the likelihood of feasibility and/or optimality.
The proposed \nnmodule{} provides the desired structural properties and it improves the performance properties.

Since we consider a highly structured problem domain, we propose to directly include the known structure into the \ac{NN} architecture. Alternatively, one could learn the structure for general problems by, e.g., attention mechanisms~\cite{Vaswani2017}. However, the related attention-based architectures usually have a high inference time which may clash with the desired fast online evaluation.

\subsection{Prediction of Binary Variables by Classification}
\label{sec:nn_classification}
As the authors in~\cite{Bertsimas2022,Abhishek2022} show, solving the prediction of binary variables as a multi-class classification problem, yields superior results than solving it as a regression problem.
Since the naive enumeration of binary assignments grows exponentially, i.e., the number of  assignments is~$2^{|\Bin|}$, where $|\Bin|$ is the number of binary variables, an effective strategy is to enumerate only combinations of binary assignments that actually appear in the data set.
While it cannot guarantee to predict all combinations, this leads to a much smaller number of possible classes and it has been observed that the resulting classifications still significantly outperform regression (cf.,~\cite{Abhishek2022}).

\subsection{Recurrent Equivariant Deep Set Architecture}
\label{sec:nn_equivariance}
The \ac{REDS} architecture, shown in Fig.~\ref{fig:nn_arch}, achieves the structural properties and improves the prediction performance.
We use training data that consists of feature-label pairs $(\Pi, \Bin)$.
The features $\Pi$ are split into obstacle-related features $\eqFeat=\{\eqFeat_0,\ldots,\eqFeat_{\dimObs-1}\}$, where $\eqFeat_i\in\R^{\dimFeatEq}$, and unstructured features $\usFeat\in \R^{\dimFeatUs}$.
The equivariant features $\eqFeat=(z_j,d_j)$ are the \emph{initial} obstacle state $z_j$ and its spatial dimension $d_j$, since all states along the horizon are predicted based on the initial state. 
The unstructured features contain all other parameters of $\Pi$, i.e., $\usFeat=(\xPmIni,\refer{v}_{\stateLon},n_\mathrm{lanes})$.

The work in~\cite{Zaheer2017} proposes a simple but effective \ac{NN} architecture that provides either permutation equivariance or invariance.
The blocks are combined in our tailored \emph{encoder layer} that maintains permutation equivariance for the equivariant outputs and invariance for the unstructured outputs, see Fig.~\ref{fig:nn_arch}.
A hidden state $\hUs\in\R^{\dimHidden}$ is propagated for unstructured features and hidden states $\hEq_j\in\R^{\dimHidden}$, with $j \in I(\dimObs-1)$, for the equivariant features, where $\hEq=[\hEq_0,\ldots,\hEq_{\dimObs-1}]^\top\in\R^{\dimObs\times\dimHidden}$. 
The encoder layer has four directions of information passing between the fixed-size unstructured and the variable-size equivariant hidden states with input dimension $\dimHidden$ and output dimension $\dimHidden'$:
\begin{enumerate}
	\item Equivariant to Equivariant:
	Equivariant deep sets~\cite{Zaheer2017} are used as layers with ReLU activation functions $\sigma(\cdot)$ and parameters $\Theta^\mathrm{ee},\Gamma^\mathrm{ee}\in\R^{\dimHidden\times\dimHidden'}$:
	\begin{equation*}
		f^\mathrm{ee}(\hEq)=\sigma(\hEq\Theta^\mathrm{ee}+\boldsymbol{1}\boldsymbol{1}^\top\hEq\Gamma^\mathrm{ee}).
	\end{equation*}
	\item Equivariant to Unstructured:
	To achieve invariance from the set elements $\hEq_j$ to the unstructured hidden state $\hUs$, the invariant layer of~\cite{Zaheer2017} is added that sums up over the set elements with parameters $\Theta^\mathrm{eu}\in\R^{\dimHidden\times\dimHidden'}$,
	\begin{equation*}
		f^\mathrm{eu}(\hEq)=\sigma\Bigg(\Big(\sum_{j=0}^{\dimObs-1}\hEq_j\Big)\Theta^\mathrm{eu}\Bigg).
	\end{equation*}
	\item Unstructured to Equivariant:
	To implement a dependency of the equivariant elements on the unstructured hidden state while maintaining equivariance, this layer equally influences each set element by 
	\begin{equation*}
		f^\mathrm{ue}(\hUs)=\sigma(\boldsymbol{1}_{\dimObs} \otimes \hUs\Theta^\mathrm{ue}),
	\end{equation*}
	with parameters $\Theta^\mathrm{ue}\in\R^{\dimHidden\times\dimHidden'}$.
	\item Unstructured to Unstructured:
	We use a standard feed-forward layer with parameters $\Theta^\mathrm{uu}\in\R^{\dimHidden\times\dimHidden'}$
	\begin{equation*}
		f^\mathrm{uu}(\hUs)=\sigma(\hUs\Theta^\mathrm{uu}).
	\end{equation*}
\end{enumerate}
\Acp{FF} act as encoders to the input features, which allows to matching the dimensions of the equivariant and unstructured hidden states.
For each layer of the \ac{REDS} in Fig.~\ref{fig:nn_arch}, the contributions are summed up to obtain the new hidden states
\begin{align*}
	{\hEq}'&=f^\mathrm{ee}(\hEq) + f^\mathrm{ue}(\hUs),\\
	{\hUs}'&=f^\mathrm{uu}(\hUs) + f^\mathrm{eu}(\hEq).
\end{align*}
A \ac{LSTM} is used as \emph{decoder} for each equivariant hidden state, transforming the hidden state into a time series of binary predictions, see Fig.~\ref{fig:nn_arch}. Another \ac{LSTM} is used as a \emph{decoder} for unstructured hidden states.
For the \ac{REDS}, the classification problem per time step (Sec.~\ref{sec:nn_classification}) needs to consider only four classes per obstacle (one for each collision-free region), cf. Fig.~\ref{fig:vehicle_shape}, and three classes for lane changes (change to the left or right, stay in lane).
\begin{comment}
\end{comment}
\begin{small}
\begin{figure}
	\begin{center}
		\fontsize{5pt}{5pt}
		\def\svgwidth{.48\textwidth}
		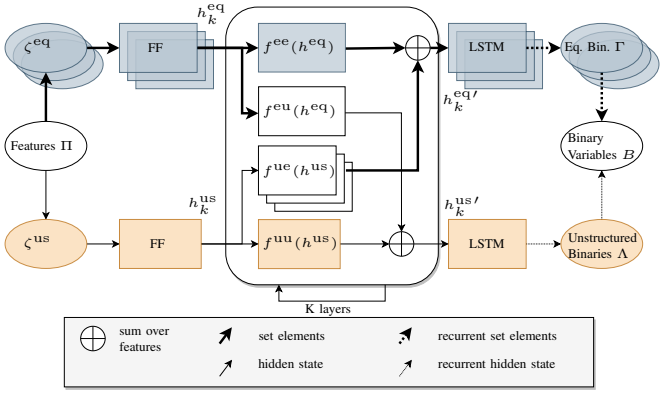
		\caption{\nndslstm{}. The blue blocks show the propagation of equivariant features, whereas the orange blocks show the propagation of unstructured features. An invariant connection couples both hidden states.}
		\label{fig:nn_arch}
	\end{center}
\end{figure}
\end{small}
\subsection{Neural Network Ensembles for Multiple Predictions}
\label{sec:nn_ensemble}
As suggested in early works~\cite{Hansen1990, Sollich1995}, using an ensemble of $\dimEnsemble$ stochastically trained \acp{NN} is a simple approach to obtain multiple guesses and improve classification accuracy.
Producing multiple predictions, the lowest cost maneuver among different candidate solutions can be selected, e.g., staying behind a vehicle or overtaking.
For typical classification tasks, no a-posteriori oracle exists that identifies the \emph{best} guess~\cite{Hansen1990}.
However, in our problem setup, the \slackqp{} solution can directly evaluate the feasibility and optimality and therefore, identify the best guess, as discussed next. 

\section{Soft \ac{QP} Solution and Selection Method}
\label{sec:slacked_QP}
For each candidate solution from the ensemble of \acp{NN}, the \slackqp{} is constructed based on the \ac{MIQP}~\eqref{eq:MIQP} with fixed binary variables and unbounded slack variables for the obstacle constraints. 
Each candidate solution consists of the binary values $\hat{\BinL}=[\hat{\binL}_0,\ldots,\hat{\binL}_{\dimHor-1}]$ and $\hat{\BinY}=[\hat{\binY}_0^0,\hat{\binY}_1^0,\ldots,\hat{\binY}_\dimHor^{\dimObs-1}]$  that are predicted by the \ac{NN}. The reference $\refer{X}_n$ in~\eqref{eq:miqp_line_ref_prog} is also fixed when the binary variables are fixed.
The resulting \slackqp{} is convex and a feasible solution always exists due to removing the upper bound for each slack variable $\sigma_i^j$, with $j\in I(\dimObs-1)$, $i \in I(\dimHor)$ and
	\begin{mini!}[3]
		{\XPm,\UPm,\Sigma}			
		{J^\mathrm{s}(\XPm,\UPm,\Sigma)\label{eq:slackqp_cost}}
		{\label{eq:slackqp}} 
		{J^\mathrm{s*}=} 
		\addConstraint{\xPm_0}{= \xPmIni, \;\; H(\XPm,\UPm)\geq0,}{\Sigma\geq 0}{},
		\addConstraint{\xPm_{i+1}}{= A\xPm_{i}+B\uPm_{i},}{i\in I(\dimHor-1)},
	\addConstraint{(\selMat\xPm_i,\sigma_i^j)}{\in \setFreeOut(d_i^j, \hat{\binY}_i^j), \quad}{i \in I(\dimHor)\nonumber},
	\addConstraint{}{}{j \in I(\dimObs-1)},
		\label{eq:slackqp_set}
	\end{mini!}
where the cost $J^\mathrm{s}(\cdot)$ is~\eqref{eq:miqp_cost}, see Tab.~\ref{table:overview_opti}.
We construct~\eqref{eq:slackqp_set} from~\eqref{eq:cobstacle_constr} by removing the safety margins $\ub{\sigma}_i^j$ and fixing the binary variables to the predicted solution guess $\hat{\BinY}$.
Therefore, the \slackqp{}~\eqref{eq:slackqp} is a relaxation of \ac{MIQP}~\eqref{eq:MIQP} with fixed integers.
Problem~\eqref{eq:slackqp} is solved for each prediction 
of the \ac{NN} ensemble, and the solution leading to the lowest cost for~\eqref{eq:slackqp_cost} is selected as output $\XNN{}$ of the module, see Fig.~\ref{fig:sw_arch}. The \slackqp{} is convex and can be solved efficiently using a structure exploiting \ac{QP} solver~\cite{Acados2021,Frison2020}.
Despite solving multiple \acp{QP} for multiple candidate solutions, the computational burden is much lower than solving an \ac{MIQP} that typically requires solving a combinatorial amount of convex relaxations.

\section{Feasibility Projection and SQP Algorithm}
\label{sec:feas_proj}
The \slackqp{} solution $(\XNN{}, \UNN{})$ may not be collision-free due to binary classification errors from the ensemble of \acp{NN}, i.e., some of the slack variables may be nonzero in the \slackqp{} solution.
In order to project the \slackqp{} solution to a collision-free trajectory, a smooth convex-concave \ac{NLP}, referred to as \safetyfilter{}, is solved in each iteration.
The \safetyfilter{} solves an optimization problem that is similar to~\eqref{eq:slackqp} (see Tab.~\ref{table:overview_opti}), but the reference trajectory is equal to the \slackqp{} solution, i.e., $\XPmRef:=\XNN{}$ and $\UPmRef:=\UNN{}$. In addition, the obstacle constraints in~\eqref{eq:slackqp_set} are replaced by smooth concave constraints based on the ellipsoidal collision region, cf. Fig.~\ref{fig:vehicle_shape}, which allows the use of an efficient SQP algorithm, e.g.,~\cite{Debrouwere2013}.

\subsection{Optimization Problem Formulation}
With the geometry parameter~$d$ of an obstacle, the inner ellipse~$\setSafe$ within~$\setOut$ is used for defining the safe-set~$\setFreeSafe$.
The ellipse center and axis matrix
\begin{align*}
	t(d)&=\frac{1}{2}[(s_\mathrm{f}+s_\mathrm{b}),(n_\mathrm{l}+n_\mathrm{r})]^\top, \\
	\Upsilon(d)&=\frac{1}{\sqrt{2}}\mathrm{diag}([s_\mathrm{f}-s_\mathrm{b},n_\mathrm{l}-n_\mathrm{r}]),
\end{align*}
are used to formulate the smooth obstacle constraint
\begin{align}
	\label{eq:collision_ellipse}
	\begin{split}
		\setFreeSafe(d) = 
		\Bigl\{(p, \xi) \Big\vert
		\norm{p-t(d)}^2_{\Upsilon^{-1}(d)}
		\geq 1-\xi\Bigr\},
	\end{split}
\end{align}
with slack variables $\xi \ge 0$. A tracking cost

\begin{equation}
	\label{eq:safetyfilter_cost}
	J^\mathrm{f}_\mathrm{tr}(\XPm,\UPm) = \sum_{i=0}^{\dimHor}\norm{\xPm_{i}-\xPmRef_{i}}_{Q}^2 + \sum_{i=0}^{\dimHor-1}\norm{\uPm_{i}-\uPmRef_{i}}_{R}^2,
\end{equation}
and a slack violation cost

\begin{equation}
	\label{eq:safetyfilter_cost_slack}
	J^\mathrm{f}_\mathrm{slack}(\Xi) = \weightHard \sum_{j=0}^{\dimObs-1}\sum_{i=0}^{\dimHor} \xi_i^j,
\end{equation}
are defined, where $\Xi=\{\xi_i^j|i\in I(N),j\in I(n_\mathrm{obs}-1)\}$. The penalty~$\weightHard\gg0$ is sufficiently large such that a feasible solution with~$\xi_i^j = 0$ can be found when it exists.
The resulting \ac{NLP} can be written as
\begin{mini!}[3]
	{\XPm,\UPm,\Xi}			
	{J^\mathrm{f}_\mathrm{tr}(\XPm,\UPm) + J^\mathrm{f}_\mathrm{slack}(\Xi)}
	{\label{eq:safetyfilter}} 
	{} % result of optimization, e.g., J^* =
	%\addConstraint{LHS.1}{RHS.1\label{Const1}}{extraConst1}
	% \breakObjective{ + \norm{x^\mathrm{F}_N-x^\mathrm{F}_{\mathrm{ref},N}}_{Q_N}^2 }
	\addConstraint{\xPm_0}{= \xPmIni, \quad H(\XPm,\UPm)\geq0,\quad}{\Xi\geq 0},
	\addConstraint{\xPm_{i+1}}{= A\xPm_{i}+B\uPm_{i},\quad}{i\in I(\dimHor-1)},
	% \addConstraint{H(\XPm,\UPm)}{\geq0}{},
	\addConstraint{(\selMat\xPm_i,\xi_i^j)}{\in \setFreeSafe(d_i^j),}{i \in I(\dimHor)\nonumber},
	\addConstraint{}{}{j \in I(\dimObs-1)},
	\label{eq:safetyfilter_concave}
\end{mini!}
using the least-squares tracking cost~\eqref{eq:safetyfilter_cost} and smooth obstacle avoidance constraints~\eqref{eq:collision_ellipse}. 
The optimal trajectory~$\XPmOpt$ of~\eqref{eq:safetyfilter} is the output of the \safetyfilter{}~$X^\mathrm{s}$ and also of the \nnmodule{} and tracked by the \ac{NMPC}, see Fig.~\ref{fig:sw_arch}. 
Since~$\setFreeOut\subseteq\setFreeSafe$, 
% considering no slack violations, 
the \ac{NLP}~\eqref{eq:safetyfilter} is a smooth relaxation of the MIQP~\eqref{eq:MIQP}.
\subsection{Application of Sequential Quadratic Programming}
\ac{NLP}~\eqref{eq:safetyfilter} has a convex-concave structure. 
Except for the concave constraints~\eqref{eq:safetyfilter_concave}, the problem can be formulated as a convex \ac{QP}.
When solving \ac{NLP}~\eqref{eq:safetyfilter} using the Gauss-Newton \ac{SQP} method~\cite{Gros2016}, this guarantees a bound on the constraint violation for the solution guess at each SQP iteration~\cite{Tran2012}.
\begin{Proposition}
	\label{thm:fp_cost_decrease}
	Consider \ac{NLP}~\eqref{eq:safetyfilter}, let $\XPm_i,\UPm_i,\Xi_i$ be the primal variables after an \ac{SQP} iteration with Gauss-Newton Hessian approximation, and let $\XPm_0,\UPm_0,\Xi_0$ be an initial guess equal to the reference, i.e., $\XPm_0=\XPmRef$, $\UPm_0=\UPmRef$. Then, the decrease in the slack cost reads
	\begin{equation}
		\label{eq:thm_1}
		J^\mathrm{f}_\mathrm{slack}(\Xi_{i}) \leq J^\mathrm{f}_\mathrm{slack}(\Xi_{0})  - J^\mathrm{f}_\mathrm{tr}(\XPm_i,\UPm_i).
	\end{equation}
\end{Proposition}
\begin{proof}
According to Lemma~4.2 of \cite{Tran2012}, cost of \eqref{eq:safetyfilter} decreases after each iteration for our problem structure, i.e.,
\begin{equation*}
	J^\mathrm{f}_\mathrm{tr}(\XPm_{i+1},\UPm_{i+1}) + J^\mathrm{f}_\mathrm{slack}(\Xi_{i+1}) \leq
	J^\mathrm{f}_\mathrm{tr}(\XPm_i,\UPm_i) + J^\mathrm{f}_\mathrm{slack}(\Xi_{i}). \nonumber
\end{equation*}
Consequently, Eq.~\eqref{eq:thm_1} can be verified, since it holds that $J^\mathrm{f}_\mathrm{tr}(\XPm_{0},\UPm_{0})=0$ due to the initialization equal to the reference and $J^\mathrm{f}_\mathrm{tr}(X,U)\geq 0, \forall X,U$, such that
\begin{align*}
	&J^\mathrm{f}_\mathrm{tr}(\XPm_i,\UPm_i) + J^\mathrm{f}_\mathrm{slack}(\Xi_{i})\leq
	J^\mathrm{f}_\mathrm{slack}(\Xi_{0}). \nonumber
\end{align*}
\end{proof}
Additionally, it can be guaranteed that the \ac{SQP} iterations remain feasible, i.e., collision-free, once a feasible solution is found, if the slack weights~$\weightHard\gg0$ are chosen \emph{sufficiently large}.
The latter requires to select a weight value such that the gradient of the \emph{exact} L1 penalty~\eqref{eq:safetyfilter_cost_slack} is larger than any gradient of the quadratic cost~\eqref{eq:safetyfilter_cost} in the bounded feasible domain. 
This property is due to the exact penalty formulation~\cite{Nocedal2006} and the inner approximations of the concave constraints \eqref{eq:safetyfilter_concave}, since a feasible linearization point in iteration~$j$ guarantees feasibility also in the next iterate~$j+1$~\cite{Debrouwere2013}.
Notably, choosing sufficiently large weights, yet not loo large to avoid ill-conditioned \acp{QP}, may be challenging in practice. Therefore, the weights could be increased in each iteration, if convergence issues were encountered~\cite{Nocedal2006}. However, we used constant weights without encountering numerical problems.

Under mild assumptions~\cite{Tran2012}, the SQP method converges to a stationary point of~\eqref{eq:safetyfilter}. The \safetyfilter{} provides a certificate of feasibility if the slack variables are zero, i.e., $\Xi_j=0$. We show that the \safetyfilter{} effectively increases the likelihood of computing a collision-free trajectory in numerical simulations.

\section{Implementation Details}
In this section, we list the most important implementation details.
Further information on the structure of the \ac{NN} model, training parameters and the \safetyfilter{} are in the Appendix.

\subsubsection{Numerical solvers}
For solving the \ac{MIQP} of the \expert{} and the \ac{QP} of the \slackqp{}, we use \texttt{gurobi}~\cite{Gurobi2023} and formulate both problems in \texttt{cvxpy}~\cite{Cvxpy2016}.
The \acp{NLP} of the \safetyfilter{}, as well as the \mpc{}, are solved by using the open-source solver \texttt{acados}~\cite{Acados2021} and the algorithmic differentiation framework \texttt{casadi}~\cite{Casadi2019}.
We use Gauss-Newton Hessian approximations, full steps, an explicit RK4 integrator and non-condensed \acp{QP} that are solved by~\texttt{HPIPM}~\cite{Frison2020}.
We use real-time iterations~\cite{Diehl2005} within the \mpc{} and limit the maximum number of \ac{SQP} iterations to~$10$ within the \safetyfilter{}.

\subsubsection{Training of the NN}
In order to formulate and train the \ac{NN}s, we use \texttt{PyTorch}. We train on datasets of~$10^5$ expert trajectories that we generate by solving the \expert{} with randomized initial parameters, sampled from an uniform distribution within the problem bounds, see Appendix~\ref{sec:appendix}. We use a learning rate of~$10^{-4}$ and a batch size of~$1024$.
The performance is evaluated on a test dataset of~$2\cdot10^3$~samples using a \emph{cross-entropy loss} and the \texttt{adam} optimizer~\cite{Adam2015}.

\subsubsection{Computations}
Simulations are executed on a \texttt{LENOVO ThinkPad L15 Gen 1} Laptop with an \texttt{Intel(R) Core(TM) i7-10510U @ 1.80GHz} CPU. The training and GPU evaluations of the \acp{NN} are performed on an Ubuntu workstation with two \texttt{GeForce RTX 2080 Ti PCI-E 3.0 11264MB GPUs}. Parts of the \nnmodule{}, namely the $\dimEnsemble{}$ \ac{NN} ensemble and the \slackqp{}, can be parallelized, which speeds up our approach by approximately the number of \acp{NN} used. 
Therefore, the \emph{serial} computation time
\begin{equation*}
t_\mathrm{s}=\sum_{i=1}^{\dimEnsemble}(t_{\mathrm{NN},i}+t_{\mathrm{QP},i})+t_\mathrm{FP},
\end{equation*}
includes each individual \ac{NN} inference time~$t_{\mathrm{NN},i}$, each \slackqp{} evaluation time~$t_{\mathrm{QP},i}$ and the \safetyfilter{} evaluation time~$t_\mathrm{FP}$.
The \emph{parallel} computation time is 
% for the \emph{parallel} architecture, computed by
\begin{equation*}
t_\mathrm{p}=\max_{i=1,\ldots,\dimEnsemble}(t_{\mathrm{NN},i}+t_{\mathrm{QP},i})+t_\mathrm{FP}.
\end{equation*}

% \section{Nonlinear Model Predictive Control}
% \label{sec:nmpc}
\subsubsection{Nonlinear Model Predictive Control}
The lower-level \mpc{} is formulated as shown in \cite{Reiter2021a}, which is similar to~\eqref{eq:safetyfilter}, but using a more detailed nonlinear kinematic vehicle model, a shorter sampling period and a shorter horizon~(see Tab.~\ref{table:overview_opti}).
The \mpc{} can be solved efficiently using the real-time iterations~\cite{Diehl2005}, based on a Gauss-Newton \ac{SQP} method~\cite{Gros2016} in combination with a structure exploiting \ac{QP} solver~\cite{Acados2021}.

\section{In-Distribution Evaluations}
\label{sec:evaluation}
First, we show the performance of the \ac{REDS} architecture, compared to the state-of-the-art architectures of~\cite{Abhishek2022}, and we demonstrate its \emph{structural properties}, see Sec.~\ref{sec:struct_properties}.
Next, we show how the \slackqp{} and the feasibility projection increase the prediction performance and the influence on the overall computation time.
In this section, we use the same distribution over the input parameters for training and testing of the \acp{NN}.

\subsection{Evaluation of the \ac{REDS}}
\label{sec:eval_nn}
We compare two variants of the proposed architecture.
First, the \ac{REDS} is evaluated, see Fig.~\ref{fig:nn_arch}.
% , which includes the \ac{LSTM} is evaluated.
Second, as an ablation study, the same architecture but without the \ac{LSTM}, referred to as \ac{EDS}, is evaluated.
In the \ac{EDS}, a \ac{FF} is used to predict the binary variables along the full prediction horizon~(see Fig.~\ref{fig:nn_arch}).
The performance is compared against the state-of-the-art architectures for similar tasks~\cite{Abhishek2022}, i.e., an \ac{LSTM} and a \ac{FF} network with a comparable amount of parameters.

The evaluation metrics are the infeasibility rate, i.e., the share of infeasible \slackqp{}s violating constraints (with nonzero slack variables) and the misclassification rate, i.e., the share of wrong classifications concerning the prediction of any binary variable.
If at least one binary variable in the prediction is wrongly classified, the whole prediction is counted as misclassified, even though the \slackqp{} computes a feasible low-cost solution.
Furthermore, we consider the suboptimality~$\costRatio$, i.e., the objective of the feasible \slackqp{} solutions~$J^{\mathrm{s}*}$ compared to the \expert{} cost~$J^{\mathrm{e}*}$,
\begin{equation}
\label{eq:suboptimallity}
    \costRatio=\frac{J^{\mathrm{s}*}-J^{\mathrm{e}*}}{J^{\mathrm{e}*}} \ge 0.
\end{equation}
A suboptimality of~$\costRatio=0$ means the cost of the prediction is equal to the optimal cost of the \expert{}.
Fig.~\ref{fig:nn_eval_nom} shows how the performance scales with the number of obstacles~$\dimObs$ and with the horizon length~$\dimHor$, without the \safetyfilter{} from Sec.~\ref{sec:feas_proj}. 
The \ac{REDS} and \ac{EDS} yield superior results to the \ac{LSTM} as soon as~$\dimObs\geq2$, and they vastly outperform the 
\ac{FF} network for an increasing number of obstacles and horizon length.
\begin{figure}
	\begin{center}
	\includegraphics[scale=0.68]{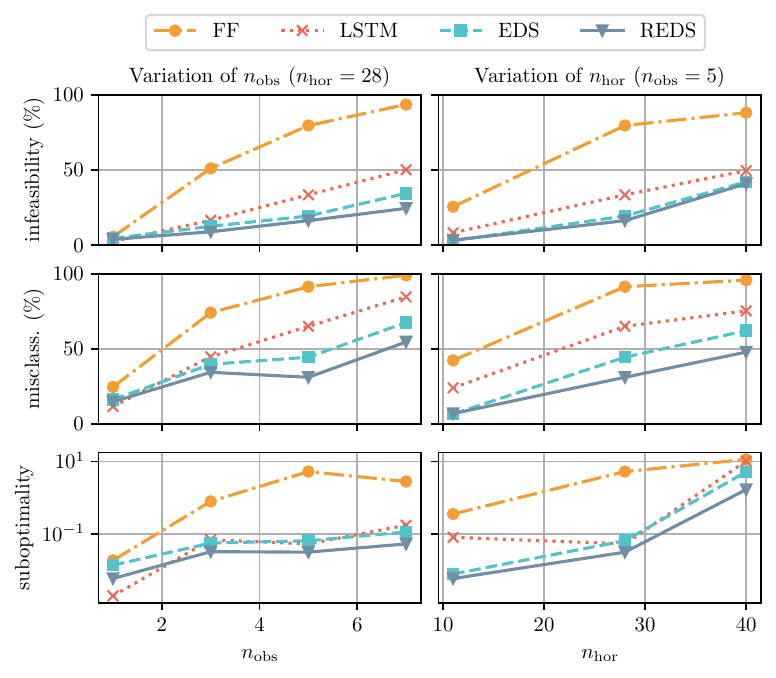}
		\caption{Performance evaluation for infeasibility rate, misclassification rate and suboptimality of different network architectures, depending on the number of obstacles and horizon length. The \nndslstm{}~outperforms the other architectures, particularly for a larger number of obstacles and a longer horizon. Suboptimality is shown in a logarithmic scale.}
		\label{fig:nn_eval_nom}
	\end{center}
\end{figure}

Besides improving the \emph{performance} metrics, the \ac{REDS} also provides the desired structural properties.
A \ac{REDS} network can be trained and evaluated with a variable number of obstacles and prediction horizon.
In Tab.~\ref{tab:interpolation}, the generalization performance of \ac{REDS} network is shown, i.e., it is evaluated for numbers of obstacles that were not present in the training data. 
Tab.~\ref{tab:interpolation} compares generalization for an \emph{interpolation} and \emph{extrapolation} of the number of obstacles and the prediction horizon.
According to Tab.~\ref{tab:interpolation}, \ac{REDS} generalizes well for samples out of the training data distribution for the number of obstacles and the horizon length.

\begin{table}
	\centering
	\ra{1.2}
	\begin{tabular}{@{}lll|ll|lll@{}}
		\addlinespace
		\toprule
		&\multicolumn{2}{c}{Training} & \multicolumn{2}{c}{Testing}& \multicolumn{3}{c}{Performance (\%)} \\
		&$\dimObs$ & $\dimHor$& $\dimObs$ & $\dimHor$ & misclass. & infeas. & subopt. \\
		\midrule
		&\multicolumn{7}{c}{Generalization for the number of obstacles: \emph{interpolation}}\\
		No general.&1-5 & 28 & 3 & 28 & 41.7 & 19.7 & 17.4\\
		General.&1,2,4,5 & 28 & 3 & 28 & \textbf{42.0} & \textbf{19.8}& \textbf{11.6}\\
		\midrule
		&\multicolumn{7}{c}{Generalization for the number of obstacles: \emph{extrapolation}}\\
		No general.&1-5 & 28 & 5 & 28 & 30.3 & 24.4 & 3.9\\
		General.&1-4 & 28 & 5 & 28 & \textbf{34.4} & \textbf{26.5} & \textbf{3.4}\\
		\midrule
		&\multicolumn{7}{c}{Generalization for the horizon length}\\
		No general.&1-3 & 16 & 1-3 & 16 & 20.4 & 8.9 & 38.1 \\
		General.&1-3 & 12 & 1-3 & 16 & \textbf{26.0} & \textbf{10.8}  & \textbf{20.9} \\
		\bottomrule
		\vspace{0.5px}
	\end{tabular}
	\caption{Evaluation of the \nndslstm{} generalization performance~(highlighted in bold).}
	\label{tab:interpolation}
\end{table}

\subsection{Evaluation for Ensemble of \ac{REDS} Networks}
\label{sec:eval_ensemble}
In Fig.~\ref{fig:nn_eval_ensemble}, we show the achieved feasibility rate on the test data using an ensemble with $\dimEnsemble=10$ \ac{REDS} networks.
In total, $2\cdot10^3$ samples are evaluated for each \ac{NN} individually and cumulatively.
The performance improves by adding more \ac{NN}s, however, also the total computation time increases.
The parallel computation time results in the maximum time over the individual networks.
\begin{figure}
	\begin{center}
		\includegraphics[scale=0.8]{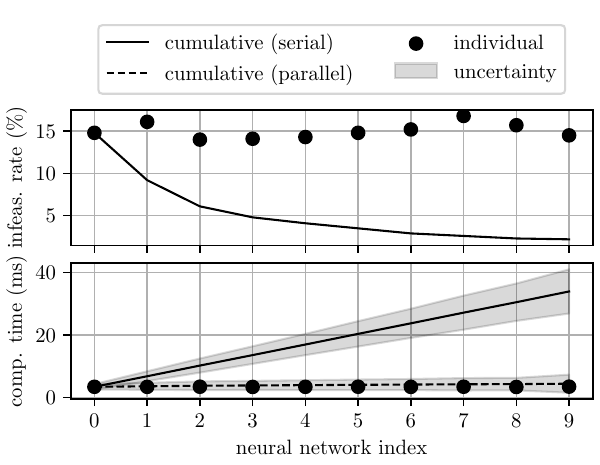}
		\caption{Individual and cumulative \ac{REDS} prediction performance (infeasibility rate) and computation time concerning the number of \acp{NN} used in the ensemble. The computation time differs for parallel and serial evaluation. For the parallel evaluation, the slowest network determines the computation time.}
		\label{fig:nn_eval_ensemble}
	\end{center}
\end{figure}
Tab.~\ref{tab:memory} shows the memory usage and the number of parameters for the different architectures. The sizes of the networks are feasible for embedded devices \cite{DiCairano2018tutorial}, i.e., the \ac{REDS} and \ac{EDS} provide improved performance with a comparable or even reduced memory footprint than \acp{FF} and \acp{LSTM}. This may be due to exploiting the structure of equivariances and invariances inherently occurring in the application. Similar observations were made in other applications where deep sets have been applied, e.g., \cite{Huegle2019}.
\begin{table}
	\centering
	\ra{1.2}
	\begin{tabular}{@{}lllll@{}}
		\addlinespace
		\toprule
		&FF&LSTM&EDS&REDS\\
		\midrule
        Memory usage (MB) &2.40&0.97&0.51&1.40\\
        Number of parameters ($10^5$)&5.98 &2.41&1.26&3.43\\
		\bottomrule
		\vspace{0.5px}
	\end{tabular}
	\caption{Memory usage and number of parameters in \ac{NN} architectures involving five obstacles and a horizon of 28 steps.}
	\label{tab:memory}
\end{table}

\subsection{Evaluation of \ac{REDS} Planner with Feasibility Projection}
\label{sec:eval_module}

\begin{figure}
	\begin{center}
		\includegraphics[scale=0.8]{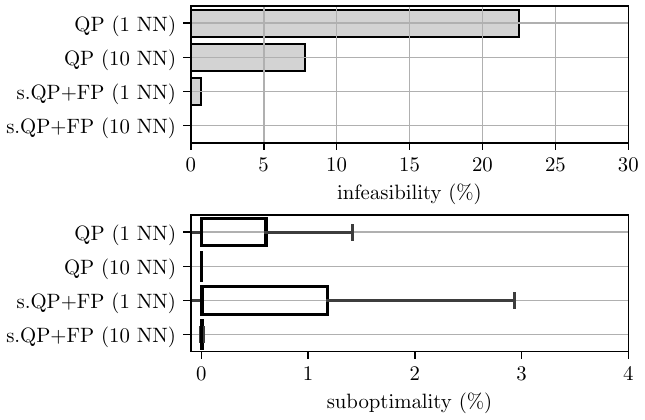}
		\caption{Open-loop comparison of infeasibility rate and suboptimality of the \nnmodule{}, using the \ac{QP} without slack variables and the \slackqp{} followed by the \safetyfilter{}. Infeasible problems are not considered in the suboptimality.}
		\label{fig:cost_open_loop}
	\end{center}
\end{figure}

The \nnmodule{} is validated on random samples of the test data, i.e., samples of the same distribution as the \ac{NN} training data.
Fig.~\ref{fig:cost_open_loop} shows the infeasibility rate and suboptimality~\eqref{eq:suboptimallity} of the \nnmodule{}, using either the \ac{QP} without slack variables, or the \slackqp{} followed by the \safetyfilter{}.
For an ensemble of ten \acp{NN}, the infeasibility rate of the \ac{QP} without slack variables is below~$10\%$ and decreased to almost~$0\%$ by using the \slackqp{} followed by the \safetyfilter{}, and the suboptimality is also negligible.
The suboptimality of the \slackqp{} followed by the \safetyfilter{} is higher than the suboptimality of the \ac{QP}, since also infeasible problems are rendered feasible, yet with increased suboptimality values.
Fig.~\ref{fig:timings_open_loop} shows box plots of the computation times related to the different components.
The main contributions to the total computation time of the \nnmodule{} stem from the \slackqp{} (median of~$\sim 4.2$ms) and the \acp{NN} (median of~$\sim 3.0$ms per network).
While the parallel architecture with ten \ac{NN}s, as well as a single \ac{NN}, decrease the worst case \ac{MIQP} computation time by a factor of approximately~$50$ and the median by a factor of~$10$, the serial approach decreases the maximum by a factor of~$8$ and the median by a factor of~$2$.
Besides computation time, the \nnmodule{} is suitable for embedded system implementation as opposed to high-performance commercial solvers like the one used here as an expert, i.e., \texttt{gurobi}~\cite{Gurobi2023}.

\begin{figure}
	\begin{center}
		\includegraphics[scale=0.8]{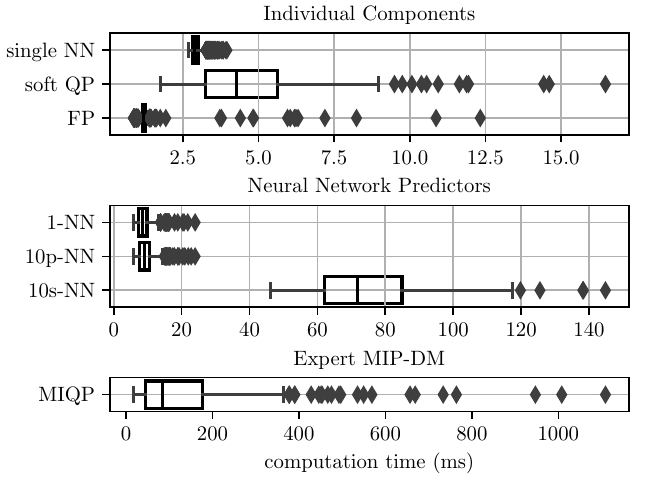}
		\caption{Box plots for open-loop comparison of the computation times of~$10^3$ samples. \nnmodule{} with an ensemble of one (1-NN) or ten \acp{NN} is parallelized (10p) or serial (10s).}
		\label{fig:timings_open_loop}
	\end{center}
\end{figure}

\section{Closed-loop Validations with \texttt{SUMO} Simulator}
\label{sec:eval_closed_loop}
The following closed-loop evaluations of the \nnmodule{} on a multi-lane highway scenario yield a more realistic performance measure. This involves challenges such as the distribution shift of the input parameters, wrong predictions and errors of the \mpc{}.

\subsection{Setup for Closed-loop Simulations}
Several choices need to be made for parameters in the \nnmodule{} and in the simulation environment.

\begin{figure*}
	\begin{center}
		\includegraphics[scale=0.6,trim={0mm 2mm 0mm 3mm},clip]{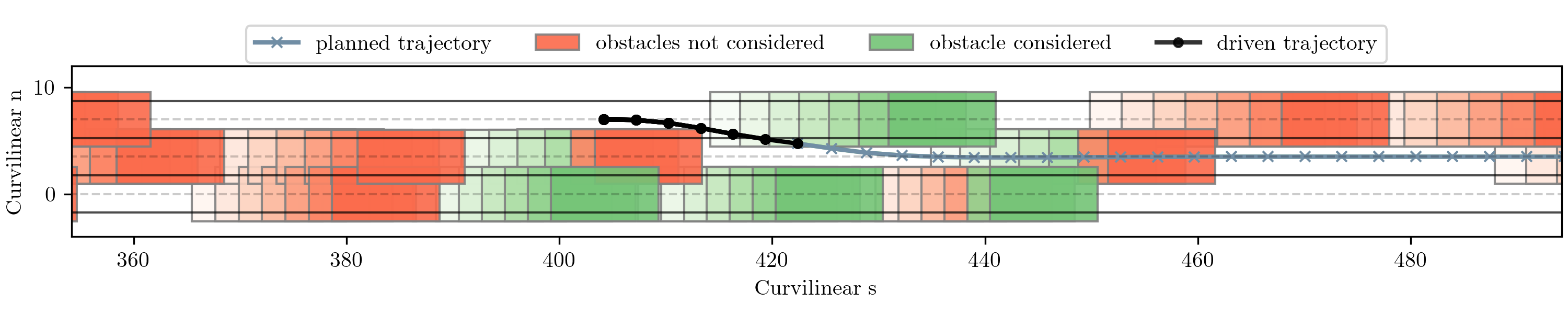}
		\caption{Obstacles considered at each planning step in the Frenet coordinate frame. The plot shows the ego point-mass trajectory and the inflated obstacles in proximity to the ego vehicle for seven time steps. The obstacle color indicates whether it was considered in planning or not.}
		\label{fig:obstacle_selection}
	\end{center}
\end{figure*}

\subsubsection{Collision avoidance}
For an environment with a large number of obstacles, we implement the following heuristic to select up to $\dimObs=5$ obstacles in proximity to the ego vehicle. We consider the closest successive obstacles in each lane that is not the lane of the ego vehicle and the leading vehicle on the ego lane.
In Fig.~\ref{fig:obstacle_selection}, the selection of obstacles in proximity to the ego vehicle is shown in the Frenet coordinate frame. Obstacles are plotted in consecutive planner time steps, and the color indicates whether they are considered at each time step.
Overtaking is allowed on both sides of a leading vehicle.

\subsubsection{Sampling frequency}
The planning frequency is set to~$5$Hz, the control and ego vehicle simulation rate is~$50$Hz and the \texttt{SUMO} simulation is~$10$Hz. 
The traffic simulator in \texttt{SUMO} is slower than the ego vehicle simulation frequency, therefore the motion of the vehicles is linearly extrapolated between \texttt{SUMO} updates.
Indeed, the \nnmodule{} is real-time capable for selected frequencies, see Req.~\ref{req:comp_t}.

\subsubsection{Vehicle models}
We use parameters of a \emph{BMW 320i} for the ego vehicle, which is a medium-sized passenger vehicle whose parameters are provided in \texttt{CommonRoad}~\cite{Althoff2017} for models of different fidelity. An \emph{odeint} integrator of \texttt{scipy} simulates the~$29$-state \emph{multi-body model}~\cite{Althoff2017} with a~$20$ms time step.
The traffic simulator \texttt{SUMO}~\cite{Sumo2018} simulates interactive driving behaviors with the \emph{Krauss model}~\cite{Krauss1998} for car following and the \emph{LC2013}~\cite{Erdmann2015} model for lane changing.
From zero up to five surrounding vehicles are selected for our comparisons.

\subsubsection{Road layout}
Standardized scenarios on German roads, provided by the \emph{scenario database} of \texttt{CommonRoad}, are fully randomized before each closed-loop simulation, including the start configuration of the ego vehicle and all other vehicles.
This initial randomization, in addition to the interactive and stochastic behavior simulated in \texttt{SUMO}, covers a wide range of traffic situations, including traffic jams, blocked lanes and irrational driver decisions such as half-completed lane changes.
The basis of our evaluations is the three-lane scenario \rdnetCol{} with all (\emph{dense}) or only a fourth (\emph{sparse}) of the vehicles in the database.

\subsection{Distribution Shift}
The generally unknown state distribution encountered during closed-loop simulations, referred to as \ac{SD}, differs from the \ac{TD}. We aim to generalize with the proposed approach to a wide range of scenarios, hence, we use the uniform distribution given in Tab.~\ref{table:sampling distribution_training} to train the \acp{NN} and in consecutive closed-loop evaluations. 
However, if the encountered \ac{SD} is known better, we propose to include \acp{NN} trained on an \emph{a priori} known \ac{SD} and include it in the ensemble of \acp{NN}.
In Fig.~\ref{fig:nn_distribution_shift}, the worse performance of an ensemble of \ac{REDS}, purely trained on the uniform \ac{TD}, followed by the \slackqp{} is shown when evaluated for samples taken from closed-loop simulations of the \rdnetCol{} scenario using the \expert{}. Additionally, Fig.~\ref{fig:nn_distribution_shift} shows \acp{NN} trained on this \ac{SD} and how they can improve the prediction performance as single networks by~$\sim10\%$ and in an ensemble by~$\sim 3\%$.
\begin{figure}
	\begin{center}
		\includegraphics[scale=0.8]{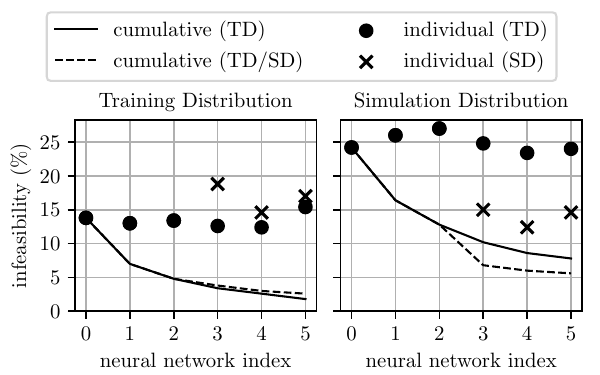}
		\caption{Individual and cumulative \ac{REDS} prediction performance for the \acf{TD} and samples encountered by an \acf{SD}. Two variants of the \ac{REDS} ensemble are compared: one purely trained on the uniform \acf{TD} and one ensemble with three \acp{NN} trained on the \ac{SD}.}
		\label{fig:nn_distribution_shift}
	\end{center}
\end{figure}

\subsection{Closed-loop Results with \texttt{SUMO} Simulator}
We compare the closed-loop performance of the \nnmodule{} with varying numbers of \acp{NN} and of the \expert{} for the dense and sparse traffic in scenario \rdnetCol{}, c.f., Fig.~\ref{fig:ensemble_sparse}.
The closed-loop cost is computed by evaluating the objective~\eqref{eq:miqp_cost} for the closed-loop trajectory and is separated into its components.
The computation times are shown for the serial and parallel evaluation of the \acp{NN}. 
Similar to the open-loop evaluation in Fig.~\ref{fig:nn_eval_ensemble}, the closed-loop results in Fig.~\ref{fig:ensemble_sparse} show a considerable performance gain when more \ac{NN}s are added to the ensemble.
Evaluating~$10$ \acp{NN} in parallel within the \nnmodule{} leads to nearly the same closed-loop performance as the \expert{}. Remarkably, a parallel computation achieves a tremendous speed-up of the worst-case computation time of approximately~$100$ times. A serial evaluation of~$10$ \acp{NN} could still be computed around~$25$ times faster for the maximum computation time compared to the \expert{}. All computation times of the \nnmodule{} variations are below the threshold of $200$ms, while the \expert{} computation time exceeds the threshold in more than~$50\%$, taking up to~$4$s for one iteration.
\begin{figure*}
	\begin{center}
		\includegraphics[scale=0.7]{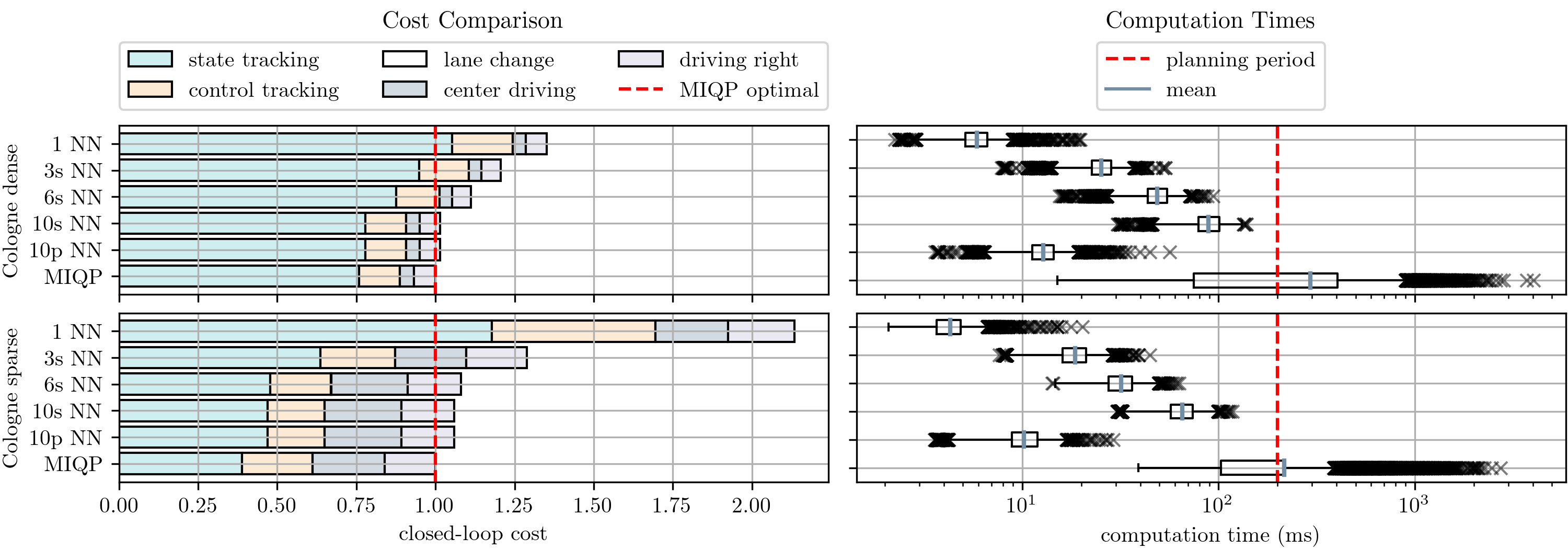}
		\caption{Performance comparison for different numbers of serial (s) and parallel (p) \acp{NN}, for dense and sparse traffic in scenario \rdnetCol{}. The closed-loop cost was computed by evaluating~\eqref{eq:miqp_cost} for the closed-loop ego trajectory and normalized against the closed-loop cost of \expert{}. The state tracking cost, including the lateral position and the desired velocity tracking error, 
  contributes the most to the chosen weights.}
		\label{fig:ensemble_sparse}
	\end{center}
\end{figure*}
Tab.~\ref{table:final_closed_loop} shows further performance metrics. 
By using more \acp{NN} within an ensemble increases the average velocity and decreases the closed-loop cost towards the \expert{} performance.
Using six or ten \acp{NN}, a single situation occurred where a leading vehicle started to change lanes but halfway decided to change back towards the original lane, resulting in a collision, which in a real situation will be avoided by an emergency (braking) maneuver.
Fig.~\ref{fig:demo_nn} shows snapshots of a randomized closed-loop \texttt{SUMO} simulation of a dense and sparse traffic scenario \rdnetCol{}. The \emph{green} colored ego vehicle successfully plans lane changes and overtaking maneuvers to avoid collisions with other vehicles~(blue), using the proposed \nnmodule{} feeding a \mpc{}. 
\begin{figure}
	\begin{center}
		\includegraphics[trim={1cm 0 1cm 0},clip,width=0.3\textwidth]{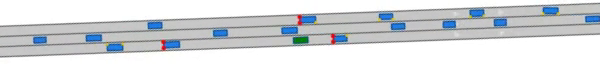}\hspace{5mm}
		\includegraphics[width=0.15\textwidth]{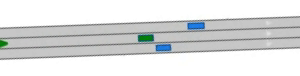}
		\includegraphics[trim={1cm 0 1cm 0},clip,width=0.3\textwidth]{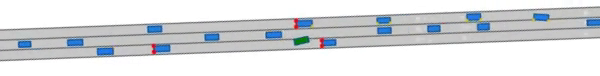}\hspace{5mm}
		\includegraphics[width=0.15\textwidth]{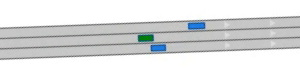}
		\includegraphics[trim={1cm 0 1cm 0},clip,width=0.3\textwidth]{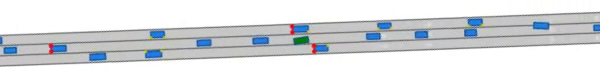}\hspace{5mm}
		\includegraphics[width=0.15\textwidth]{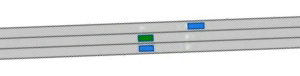}
		\includegraphics[trim={1cm 0 1cm 0},clip,width=0.3\textwidth]{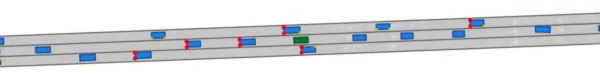}\hspace{5mm}
		\includegraphics[width=0.15\textwidth]{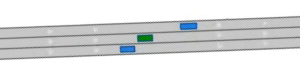}
		\includegraphics[trim={1cm 0 1cm 0},clip,width=0.3\textwidth]{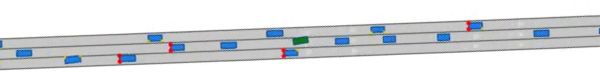}\hspace{5mm}
		\includegraphics[width=0.15\textwidth]{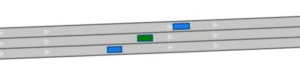}
		\includegraphics[trim={1cm 0 1cm 0},clip,width=0.3\textwidth]{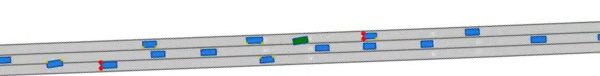}\hspace{5mm}
		\includegraphics[width=0.15\textwidth]{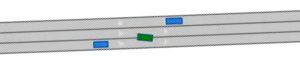}
		\includegraphics[trim={1cm 0 1cm 0},clip,width=0.3\textwidth]{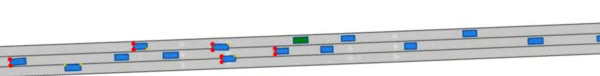}\hspace{5mm}
		\includegraphics[width=0.15\textwidth]{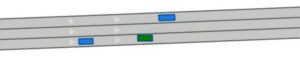}
		\caption{Snapshots of simulated traffic scenario \rdnetCol{} dense (left) and sparse (right) with \texttt{SUMO} and \texttt{CommonRoad}, showing the ego vehicle~(green) and other vehicles~(blue). The \nnmodule{} is real-time feasible and used in combination with \mpc{}, see Fig.~\ref{fig:sw_arch}. Red light at the rear of the vehicle indicate braking.}
		\label{fig:demo_nn}
	\end{center}
\end{figure}

\begin{table}
	\centering
	\ra{1.2}
	\begin{tabular}{@{}lllllll@{}}
		\addlinespace
		\toprule
	\multicolumn{2}{c}{} & \multicolumn{5}{c}{\rdnetCol{}-dense} \\
Property 	 & Unit 	  & MIQP & NN-1 & NN-3 & NN-6 & NN-10\\
\midrule
collisions	&		&	0	&	0	&	0	&	1	&	1\\
vel. mean	&	 $\frac{m}{s}$	&	13.10	&	12.71	&	12.91	&	13.00	&	13.07\\
% vel. max	&	 $\frac{m}{s}$	&	16.85	&	16.85	&	16.85	&	16.85	&	16.85\\
vel. min	&	 $\frac{m}{s}$	&	0.41	&	1.17	&	0.00	&	0.00	&	0.00\\
lane changes	&		&	457	&	431	&	469	&	471	&	462\\
cost	&	$\frac{1}{s}$	&	49.48	&	66.89	&	59.68	&	55.02	&	50.22\\
%ctrl. success	&	\%	&	100.00	&	100.00	&	100.00	&	100.00	&	100.00\\
		\midrule
		\multicolumn{2}{c}{} & \multicolumn{5}{c}{\rdnetCol{}-sparse} \\
Property 	 & Unit 	  & MIQP & NN-1 & NN-3 & NN-6 & NN-10\\
\midrule
collisions	&		&	0	&	0	&	0	&	0	&	0\\
vel. mean	&	 $\frac{m}{s}$	&	13.94	&	13.69	&	13.83	&	13.87	&	13.89\\
% vel. max	&	 $\frac{m}{s}$	&	17.20	&	16.85	&	16.85	&	16.85	&	16.85\\
vel. min	&	 $\frac{m}{s}$	&	7.10	&	7.10	&	7.10	&	7.10	&	7.10\\
lane changes	&		&	354	&	337	&	337	&	353	&	353\\
cost	&	$\frac{1}{s}$	&	5.81	&	12.40	&	7.49	&	6.27	&	6.16\\
% ctrl. success	&	\%	&	100.00	&	100.00	&	100.00	&	100.00	&	100.00\\
		\bottomrule
		\vspace{0.5px}
	\end{tabular}
	\caption{Closed-loop evaluation for scenario \rdnetCol{} with dense and sparse traffic.}
	\label{table:final_closed_loop}
\end{table}

\section{Conclusions and Discussion}
\label{sec:conclusion}
We proposed a supervised learning approach for achieving real-time feasibility for mixed-integer motion planning problems. 
Several concepts are introduced to achieve a nearly optimal closed-loop performance when compared to an expert MIQP planner. First, it was shown that inducing structural problem properties such as invariance, equivariance and recurrence into the \ac{NN} architecture improves the prediction performance among other useful properties such as generalization to unseen data. 
Secondly, the \slackqp{} can correct wrong predictions, inevitably linked to the \ac{NN} predictions and are able to evaluate an open-loop cost. This favors a parallel architecture of an ensemble of predictions and \slackqp{} computations to choose the lowest-cost trajectory. In our experiments, adding \acs{NN} to the ensemble improved the performance considerably and monotonously. This leads to the conclusion that \acp{NN} may be added as long as the computation time is below the planning threshold and as long as parallel resources are available. To further promote safety, an \ac{NLP}, i.e., the \safetyfilter{}, is used to plan a collision-free trajectory. The computational burden of the \ac{NLP} is small, compared to the \expert{} since it optimizes the trajectory only locally and therefore, omits the combinatorial variables.

%Discussion
Although we have evaluated the proposed approach for multi-lane traffic, the application to other urban driving scenarios is expected to perform similarly for a similar number of problem parameters due to the following. Many works, e.g.,~\cite{Kessler2020, Kessler2023,Quirynen2023}, use similar MIQP formulations for a variety of \ac{AD} scenarios, including traffic lights, blocked lanes and merging. The formulations mainly differ in the specific environment. Many scenario specifics can be modeled by using obstacles and constraints related to the current lane~\cite{Quirynen2023}, both considered in the presented approach. However, with an increasing number of problem parameters and rare events, the prediction of binary variables may become more challenging.

\bibliographystyle{IEEEtran}
\bibliography{lib}


%\bf{If you include a photo:}\vspace{-33pt}
\begin{IEEEbiography}[{\includegraphics[width=1in,height=1.25in,clip,keepaspectratio]{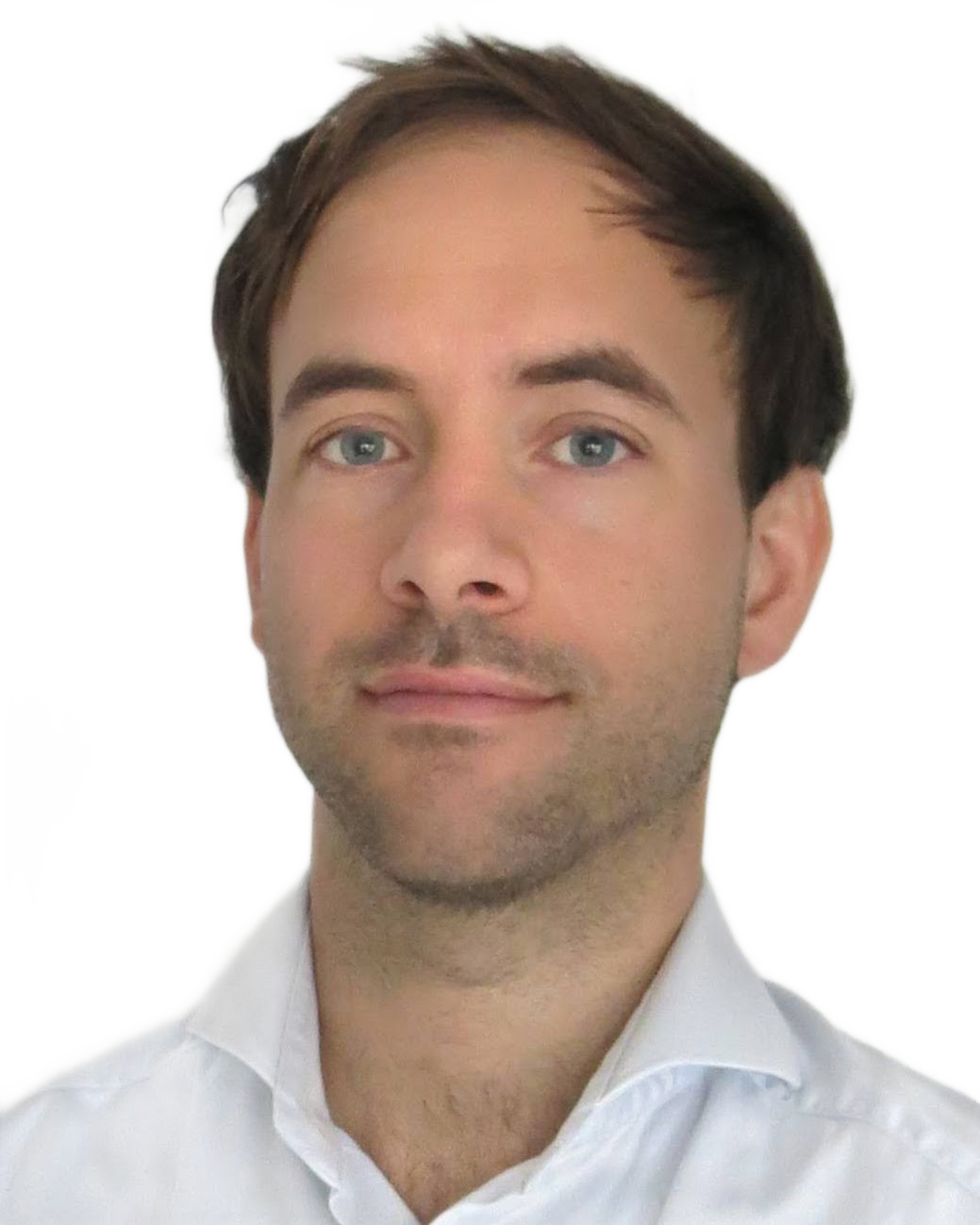}}]{Rudolf Reiter} received his Master's degree in 2016 in electrical engineering with a focus on control systems from the Graz University of Technology, Austria. From 2016 to 2018 he worked as a Control Systems Specialist at the Anton Paar GmbH, Graz, Austria. From 2018 to 2021 he worked as a researcher for the Virtual Vehicle Research Center, in Graz, Austria, where he started his Ph.D. in 2020 under the supervision of Prof. Dr. Mortiz Diehl. Since 2021, he continued his Ph.D. at a Marie-Skłodowska Curie Innovative Training Network position at the University of Freiburg, Germany. His research focus is within the field of learning- and optimization-based motion planning and control for autonomous vehicles and he is an active member of the Autonomous Racing Graz team.

\end{IEEEbiography}
\begin{IEEEbiography}[{\includegraphics[width=1in,height=1.25in,clip,keepaspectratio]{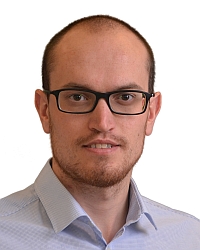}}]{Rien Quirynen} received the Bachelor’s degree in computer science and electrical engineering and the Master’s degree in mathematical engineering from KU Leuven, Belgium. He received a four-year Ph.D. Scholarship from the Research Foundation–Flanders~(FWO) in 2012-2016, and the joint Ph.D. degree from KU Leuven, Belgium and the University of Freiburg, Germany. He worked as a senior research scientist at Mitsubishi Electric Research Laboratories in Cambridge, MA, USA from early 2017 until late 2023. Currently, he is a staff software engineer at Stack AV. His research focuses on numerical optimization algorithms for decision making, motion planning and predictive control of autonomous systems. He has authored/coauthored more than 75 peer-reviewed papers in journals and conference proceedings and 25 patents. Dr. Quirynen serves as an Associate Editor for the Wiley journal of Optimal Control Applications and Methods and for the IEEE CCTA Editorial Board.
\end{IEEEbiography}
\begin{IEEEbiography}[{\includegraphics[width=1in,height=1.25in,clip,keepaspectratio]{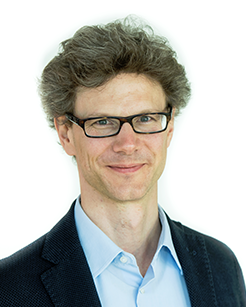}}]{Moritz Diehl}
	studied mathematics and physics at Heidelberg University, Germany and Cambridge University, U.K., and received the Ph.D. degree in optimization and nonlinear model predictive control from the Interdisciplinary Center for Scientific Computing, Heidelberg University, in 2001. From 2006 to 2013, he was a professor with the Department of Electrical Engineering, KU Leuven University Belgium. Since 2013, he is a professor at the University of Freiburg, Germany, where he heads the Systems Control and Optimization Laboratory, Department of Microsystems Engineering (IMTEK), and is also with the Department of Mathematics. His research interests include optimization and control, spanning from numerical method development to applications in different branches of engineering, with a focus on embedded and on renewable energy systems.
\end{IEEEbiography}
\begin{IEEEbiography}[{\includegraphics[width=1in,height=1.25in,clip,keepaspectratio]{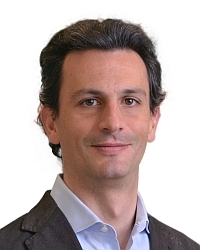}}]{Stefano Di Cairano} %(M'??--SM'??)
	received the Master’s (Laurea) and the Ph.D. degrees in information engineering in 2004 and 2008, respectively, from the University of Siena, Italy. 
	During 2008-2011, he was with Powertrain Control R\&A, Ford Research and Advanced Engineering, Dearborn, MI, USA. Since 2011, he is with Mitsubishi Electric Research Laboratories, Cambridge, MA, USA, where he is currently a Deputy Director, and a Distinguished Research Scientist. His research focuses on optimization-based control and decision-making strategies for complex mechatronic systems, in automotive, factory automation, transportation systems, and aerospace. His research interests include model predictive control, constrained control, path planning, hybrid systems, optimization, and particle filtering. He has authored/coauthored more than 200 peer-reviewed papers in journals and conference proceedings and 70 patents.
	Dr. Di Cairano was the Chair of the IEEE CSS Technical Committee on Automotive Controls and of the IEEE CSS Standing Committee on Standards. He is the inaugural Chair of the IEEE CCTA Editorial Board and was an Associate Editor of the IEEE \sc{Transactions on Control Systems Technology}. 
\end{IEEEbiography}

\appendix
\label{sec:appendix}
In the following, we define the most important parameters used in the numerical simulations of this paper. 
For the closed-loop simulations in \texttt{SUMO}, we used the parameter values in Tab.~\ref{table:app_nn_closedloop}. The reference velocity was set higher than the average nominal velocity to cause more overtaking maneuvers.
For the optimization problems, we used the parameters shown in Tab.~\ref{table:app_opti_parameters}, in addition to the values presented in Tab.~\ref{table:overview_opti}.
In Tab.~\ref{table:app_nn_parameters}, neural network hyperparameters are shown for architectures \ac{FF}, \ac{LSTM}, \ac{EDS} and \ac{REDS}. 
The same \ac{LSTM} layer is used for each set element (i.e., for each obstacle in our case) to achieve equivariance in the \ac{REDS} network.

\begin{table}
	\centering
	\ra{1.2}
	\begin{tabular}{@{}lll@{}}
		\addlinespace
		\toprule
		Category & Parameter & Value\\
        \midrule
        General & episode length & $30$s\\
        & nominal road velocity & $13.9\frac{\mathrm{m}}{\mathrm{s}}$\\
        & maximum number of lanes & 3\\
        & lanes width $ d_\mathrm{lane}$& $3.5$m\\
        & vehicle lengths & $5.39$m\\
        & vehicle widths & $2.07$m\\
        & ego desired velocity & $15\frac{\mathrm{m}}{\mathrm{s}}$\\
        & maximum obstacle velocity &  $23\frac{\mathrm{m}}{\mathrm{s}}$\\
        & minimum obstacle velocity &  $0.2\frac{\mathrm{m}}{\mathrm{s}}$\\
        \midrule
        Dense scenario & traffic flow  & $0.56\frac{\mathrm{vehicles}}{\mathrm{lane}\cdot \mathrm{s}}$ \\
        & traffic density & $0.04\frac{\mathrm{vehicles}}{\mathrm{lane}\cdot \mathrm{m}}$ \\
        \midrule
        Sparse scenario & traffic flow  & $0.13\frac{\mathrm{vehicles}}{\mathrm{lane}\cdot \mathrm{s}}$ \\
        & traffic density & $0.01\frac{\mathrm{vehicles}}{\mathrm{lane}\cdot \mathrm{m}}$ \\
    	\midrule
            \safetyfilter{} & diag$(Q)$ & $[1,1,1,1]^\top$\\
            & diag$(R)$ & $[1, 1]^\top$\\
            & slack weight $\weightHard$ & $10^6$\\
      \bottomrule
		\vspace{0.5px}
	\end{tabular}
	\caption{Parameters for closed-loop simulations in \texttt{SUMO}.}
	\label{table:app_nn_closedloop}
\end{table}

\begin{table}
	\centering
	\ra{1.2}
	\begin{tabular}{@{}lll@{}}
		\addlinespace
		\toprule
		Category & Parameter & Value\\
        \midrule
        \expert{} & diag$(Q)$ & $[0,14, 10, 1]^\top$\\
        & diag$(R)$ & $[4, 0.5]^\top$\\
        & lane change weight $w_\mathrm{lc}$ & $3\cdot10^3$\\
        & side preference weight $w_\mathrm{rght}$ & $3$\\
        & safe distances $[\sigma_\mathrm{f},\sigma_\mathrm{b}, \sigma_\mathrm{l}, \sigma_\mathrm{r}]^\top$ & $[0.5,12,0.5,0.5]^\top$m\\
        & minimum controls $\lb{u}$& $[\text{-}10,\text{-}5]^\top\frac{\mathrm{m}}{\mathrm{s}}$\\
        & maximum controls $\ub{u}$& $[3,5]^\top\frac{\mathrm{m}}{\mathrm{s}}$\\
        & velocity ratio constraint $\alpha$& $0.3$\\
		\bottomrule
		\vspace{0.5px}
	\end{tabular}
	\caption{Parameters in \ac{MIQP} formulation~\eqref{eq:MIQP} for \expert{}.}
	\label{table:app_opti_parameters}
\end{table}

\begin{table}
	\centering
	\ra{1.2}
	\begin{tabular}{@{}llll@{}}
		\addlinespace
		\toprule
		Par. & Range & Par. & Range\\
		\midrule
		lat. pos. &$[0, n_\mathrm{lanes}d_\mathrm{lanes}]-\frac{d_\mathrm{lanes}}{2}$ &lanes $ $& [1,3]\\
		obs. lon. pos. & $[\text{-}120,200]\mathrm{m}$&lon. vel. & $[0,30]\frac{\mathrm{m}}{\mathrm{s}}$ \\
		lat. vel. & $[\text{-}1,1]\frac{\mathrm{m}}{\mathrm{s}}$ &obstacles &[1,5]\\
		\bottomrule
		\vspace{0.5px}
	\end{tabular}
	\caption{Ranges of uniform training data distributions.}
	\label{table:sampling distribution_training}
\end{table}

\begin{table}
	\centering
	\ra{1.2}
	\begin{tabular}{@{}lll@{}}
		\addlinespace
		\toprule
		Category & Parameter & Value\\
        \midrule
        General & activation function & ReLU\\
            & batch size & $128$ \\
            & step size& $5\cdot 10^{\text{-}5}$\\
            & epochs & $1500$\\
            & optimizer & \texttt{adam}\\
            & loss function & cross-entropy \\
            & weight decay & $10^{-5}$\\
            & training samples & $10^{5}$\\
            & test samples & $10^{3}$\\
		\midrule
		FF & hidden layers & $7$ \\
           & neurons per layer & $128$ \\
           \midrule
        LSTM & layers & $2$ \\
           & hidden network size & $128$ \\
           & input network layers & $2$ \\
           & output network layers & $2$ \\
           \midrule
        EDS & layers & $7$ \\
           & equivariant hidden size & $64$ \\
           & unstructured hidden size & $64$ \\
           & input network layers & $2$ \\
           & input network hidden size & $128$ \\
           & output network layers & $2$ \\
           & output network hidden size & $128$ \\
           \midrule
        REDS & layers & $7$ \\
           & hidden network sizes & $64$ \\
           & input network layers & $1$ \\
           & output network layers & $1$ \\
           & eq. LSTM output network layers & $1$ \\
           & unstr. LSTM output network layers & $1$ \\
		\bottomrule
		\vspace{0.5px}
	\end{tabular}
	\caption{Hyperparameters for the \ac{NN} architectures.}
	\label{table:app_nn_parameters}
\end{table}

% \vfill
\end{document}